\documentclass[conference]{IEEEtran}
\usepackage{times}
\usepackage{amsmath,amsfonts}
\usepackage{amsthm}
\usepackage{mathtools}

\newtheorem{theorem}{Theorem}

\newtheorem{definition}{Definition}
\usepackage[numbers]{natbib}
\usepackage{multicol}
\usepackage{multirow}
\usepackage[bookmarks=true]{hyperref}
\usepackage{balance}

\DeclareMathOperator*{\argmin}{arg\,min}

\def\ergMetr{\mathcal{E}}

\def\Re{\mathbb{R}}
\def\expSpace{\mathcal{W}}

\title{Time Optimal Ergodic Search}
\author{\IEEEauthorblockN{Dayi Dong, Henry Berger, and Ian Abraham}
\IEEEauthorblockA{Yale University, New Haven, CT, USA  
\\ Email: \{ethan.dong, henry.berger, ian.abraham\}@yale.edu }
}

\begin{document}

\maketitle

\begin{abstract}
    Robots with the ability to balance time against the thoroughness of search have the potential to provide time-critical assistance in applications such as search and rescue. 
    Current advances in ergodic coverage-based search methods have enabled robots to completely explore and search an area in a fixed amount of time. 
    However, optimizing time against the quality of autonomous ergodic search has yet to be demonstrated.
    In this paper, we investigate solutions to the time-optimal ergodic search problem for fast and adaptive robotic search and exploration. 
    We pose the problem as a minimum time problem with an ergodic inequality constraint whose upper bound regulates and balances the granularity of search against time. 
    Solutions to the problem are presented analytically using Pontryagin's conditions of optimality and demonstrated numerically through a direct transcription optimization approach. 
    We show the efficacy of the approach in generating time-optimal ergodic search trajectories in simulation and with drone experiments in a cluttered environment. 
    Obstacle avoidance is shown to be readily integrated into our formulation, and we perform ablation studies that investigate parameter dependence on optimized time and trajectory sensitivity for search. 
\end{abstract}

\IEEEpeerreviewmaketitle

\section{Introduction}

    The ability for robots to effectively balance time against the thoroughness of search in strict time conditions is vital for providing timely assistance in many search and rescue applications~\cite{adams2007search, mayer2019drones}.
    For example, it is often desired to have robots quickly survey large areas in minimal time and then execute a refined search based on any information gathered. 
    This approach can better assist rescue personnel in providing immediate assistance as needed.
    While recent algorithmic advances have made it possible to generate robot trajectories that provide effective coverage of search areas~\cite{prabhakar2020ergodic, Chen-RSS-22, tranzatto2022cerberus}, few consider the explicit dependence on time in the problem. 
    What makes the problem of reasoning about time versus coverage difficult is in the inherent duality between time spent covering an area and the thoroughness of the coverage. 
    Therefore, in this work, we are interested in addressing the question: is it possible to balance time and coverage quality in a single optimization problem?
    \begin{figure}
        \centering
        \includegraphics[width=\linewidth]{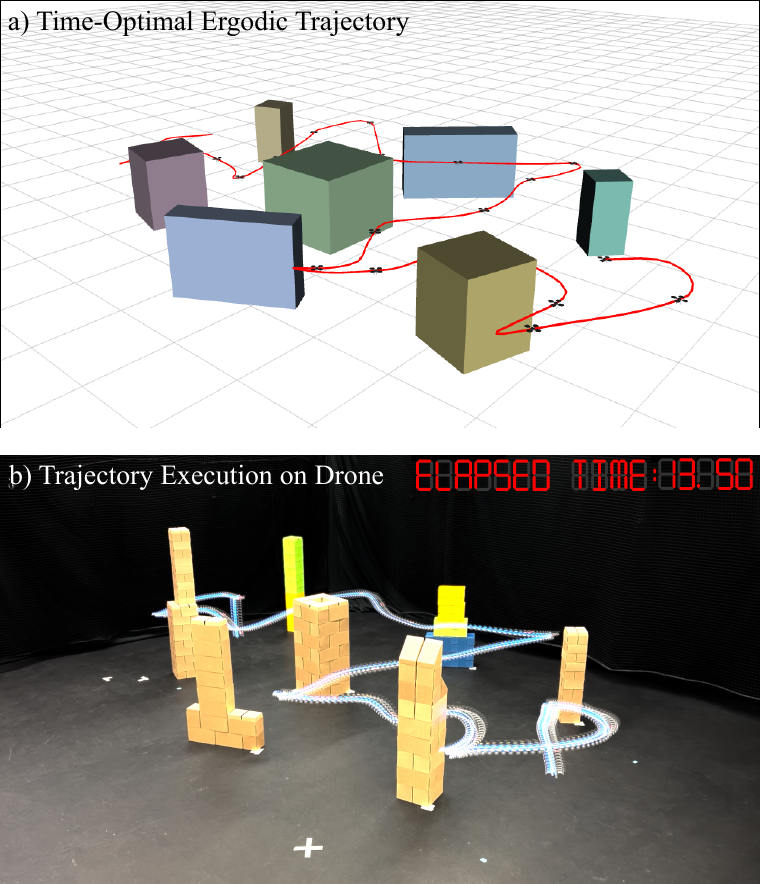}
        \caption{\textbf{Example Time-Optimal Ergodic Search Trajectories.} The proposed work investigates solutions to the time-optimal ergodic search problem for generating time-optimal coverage trajectories for search and exploration. a) Planned time-optimal trajectory for coverage in a cluttered environment in optimal time. b) Experimental drone trajectory execution of time-optimal ergodic trajectory through the cluttered environment. Trajectory was optimized to uniformly explore the environment in $13.5$s. Multimedia demonstration provided in \url{https://sites.google.com/view/time-optimal-ergodic-search} and code \url{https://github.com/ialab-yale/time_optimal_ergodic_search} .}
        \label{fig:1}
    \end{figure}

    Autonomous search and exploration has largely been studied from the perspective of coverage-based methods~\cite{galceran2013survey, dai2018quality, zelinsky1993planning, fazli2010complete}.
    These problems optimize a path that a robot will follow that visits a discretized set of nodes (or way-points) defined over a work-space (i.e., search area).
    Similar problems exist in continuous spaces and are solved through some form of spatial approximation~\cite{siligardi2019robust, pratissoli2022coverage} or coverage based on sensor envelop~\cite{breitenmoser2010voronoi, 9982287} with the use of multiple robots. 
    However, few works include time considerations, i.e. how long a robot spends in an area and how quickly the robot navigates and explores a space. 
    Methods that do consider time will often do so in bi-level optimization or as hybrid approaches that still require some form of node-based discretization~\cite{nenchev2013towards, klesh2008real}. 
    In contrast, recent advances in ergodic coverage-based search methods have demonstrated it is possible to consider time more explicitly in autonomous coverage problems~\cite{mathew2011metrics, miller2013trajectory, coffin2022multi, abraham2021ergodic, lerch2022safety, abraham2018decentralized}. 

    Ergodic search methods optimize continuous robot search trajectories by minimizing the distance between how long a robot spends in a given region and a measure of information distributed in the region~\cite{mathew2011metrics, miller2015ergodic}.
    This distance is measured using the ergodic metric~\cite{mathew2011metrics, scott2009capturing} which can be directly optimized against robot trajectories and arbitrary measures of ``information''. 
    As a result, ergodic trajectories spend more time in high-information areas while quickly exploring in low-information regions given enough time~\cite{mathew2011metrics, patel2021multi, scott2009capturing, abraham2021ergodic, Abraham-RSS-18}. 
    However, these methods optimize trajectories over a fixed time horizon, resulting in a lack of control over the granularity of how a robot searches an area. 
    Therefore, in this paper, we pose and investigate solutions to the time-optimal ergodic search problem for generating time-optimal robotic search trajectories that sufficiently explores an area.

    This paper proposes a trajectory optimization routine for scenarios where robots need to generate dynamic trajectories that optimize continuous coverage in minimum time. 
    Our approach is to formulate this problem as a time-optimal ergodic search problem where the ergodic metric imposes a coverage constraint.
    Satisfying the ergodic metric constraint is shown to yield sufficient coverage requirements based on an upper bound value that can be afforded by the robot~\cite{scott2009capturing}. 
    Because the metric is defined over Fourier spectral modes, a constraint permits optimizing time against trajectories that provide varying levels of continuous coverage over a space. 
    We investigate computing trajectory solutions to the proposed time-optimal ergodic search problem by 1) analytically demonstrating the existence of conditions of optimality based on Pontryagin's maximum principle~\cite{kopp1962pontryagin}, and 2) numerically using a direct transcription-based approach~\cite{foehn2021time, dal2019comparison, lerch2022safety}.
    Furthermore, we demonstrate time-optimal trajectories in simulation and in drone experiments in cluttered environments through the integration of safety-based obstacle avoidance constraints~\cite{ames2019control, agrawal2017discrete, lerch2022safety}. 
    In summary, our contributions are as follows: 
    \begin{enumerate}
        \item A novel time-optimal ergodic trajectory optimization method for producing time-optimal coverage trajectories for autonomous search and exploration; 
        \item Proof of analytical conditions of optimality for the time-optimal ergodic search problem; and 
        \item Demonstration of time-optimal search trajectories on a drone system in a cluttered environment (see Fig.~\ref{fig:1}). 
    \end{enumerate}

    The paper is structured as follows: Section~\ref{sec:related} overviews related work. Section~\ref{sec:prelims} describes preliminary information on ergodic search and time-optimal control. Section~\ref{sec:time_opt_erg} poses the time-optimal ergodic search problem and presents solutions to the problem. Section~\ref{sec:results} then presents various simulated and experimental results for the proposed solution to generate time-optimal ergodic search trajectories. Last, Section~\ref{sec:conclusion} provides conclusions and an outlook on future work.

\section{Related Work} \label{sec:related}

    \subsection{Coverage-Based and Ergodic Search Methods}
        
        Prior work on coverage-based planning for search and exploration has largely been focused on specifying paths or assigning robots to locations that maximizes sensor coverage in a bounded space~\cite{galceran2013survey,choset2001coverage}.
        These solutions provide guaranteed coverage over a grid and provide robot paths using algorithms such as lawn mower algorithms~\cite{araujo2013multiple,dai2018quality, bahnemann2021revisiting, cabreira2019survey} and traveling sales-person problems~\cite{ulrich1997autonomous, applegate2011traveling, bahnemann2021revisiting}. 
        Recent extensions have moved away from the limited grid approximations and worked on continuous work-spaces using cellular decomposition or continuous potential-field methods~\cite{howard2002mobile, breitenmoser2010voronoi, kim2006local}.
        In addition, information-based extensions to these search methods have provided effective strategies for robots exploring unstructured areas~\cite{Chen-RSS-22, paull2012sensor, li2020high}. 
        However, as search-based methods moved towards continuous spaces, coverage guarantees become more difficult to obtain, especially under the presence of distributed information. 

        Novel work on ergodic search methods has emerged to compute continuous coverage trajectories of an area given enough time~\cite{mathew2011metrics}. 
        Ergodic search methods optimize robot trajectories against some underlying distributed information over an area which the robot can explore~\cite{patel2021multi, miller2015ergodic, coffin2022multi, prabhakar2020ergodic}. 
        The success of ergodic search methods compared to prior methods is attributed to the unique ergodic metric used in the trajectory optimization. 
        The ergodic metric quantifies the effectiveness of a trajectory in exploring a region based on the time a robot spends in an area. 
        A trajectory is ergodic (i.e., optimizes the ergodic metric) if the time spent along the trajectory in each region is proportional to the measure of information distributed in that region.
        As a result, optimized ergodic trajectories are significantly more robust to external sensor disturbances~\cite{miller2015ergodic} and have been shown to be an optimal strategy for information-gathering tasks~\cite{silverman2013optimal, dressel_optimality_2018} with real-world application~\cite{prabhakar2020ergodic}.
        However, prior work typically has planning horizons that are fixed and are not considered part of the optimization. 
        In this work, we investigate the time-optimal extension of the ergodic search problem and its solutions.

    \subsection{Time Optimal Planning and Control}
    
        Time-optimality in planning and control is a well-studied problem going back to the original statement of Pontryagin's maximum principle~\cite{kopp1962pontryagin}. 
        The canonical problem minimizes time subject to continuous-time system dynamics and constraints on the state of the system. 
        For select systems, the conditions of optimality generate a closed-form control solution to reach desired states in optimal time~\cite{lasalle2016time, kopp1962pontryagin}. 
        Recent research on time-optimal planning has since extended the work for fast drone flight at the level of human drone pilots~\cite{foehn2021time, romero2022time}. 
        These methods solve time-optimal trajectories over specified control ``knot'' points~\cite{foehn2021time} given a trajectory tracking cost. 
        In this work, we use a variation of the solution in~\cite{foehn2021time} to directly compute ergodic trajectory solutions from our time-optimal ergodic search problem formulation.

        With respect to search and exploration, past work on time-optimal search has typically used various ``hybrid'' formulations to solve the problem of trajectory generation~\cite{reynolds2001hybrid}.
        These approaches restrict robot trajectories on discrete node-based structures and then optimize for time along each node, which takes the form of Shortest Watchman Tour problems~\cite{chin1986optimum}, Art Gallery problems \cite{o1987art}, and Traveling Sales-person problems~\cite{applegate2011traveling}. 
        Other time-optimal trajectory planning methods have used a two-step optimization approach that first generates a motion path \cite{tian2004effective,wang2020robot} and then refines the time of the found path \cite{baghli2017optimization,kim2015trajectory,du2022time}. However, these methods do not consider search in the problem, nor do they consider the coupling between continuous trajectory planning, time, and the physical robot's dynamic constraints.

\section{Preliminaries} \label{sec:prelims}

    In this section, we provide an overview of ergodic search methods and outline the necessary nomenclature used throughout this paper. The canonical ergodic trajectory optimization problem is formulated, and then we briefly define the time-optimal control problem statement for reference. 
    
    \subsection{Ergodic Search}

        Let us first define a robot trajectory at time $t$ with state $x(t) : \Re^+ \rightarrow \mathcal{X} \subseteq \mathbb{R}^n$ and control input $u(t) : \Re^+ \rightarrow \mathcal{U} \subseteq \Re^m$ where $\mathcal{X}, \mathcal{U}$ are the state and control spaces of dimensionality $n$ and $m$ respectively. 
        Next, define $\dot{x} = f(x(t), u(t))$ where $f(x,u) : \mathcal{X} \times \mathcal{U} \to \mathcal{T}_\mathcal{X}$ is the continuous-time (potentially nonlinear) dynamics of the robot.
        In addition, let us define a map $g(x) : \mathcal{X} \to \expSpace$ such that $\expSpace= [0,L_0]\times \ldots \times [0, L_{v-1}]$ where $v\leq n$, and $L_i$ are the bounds of the workspace $\expSpace$ which we denote as the exploration space.\footnote{For example, $g(x) = \mathbf{I}_p x$ where $\mathbf{I}_p$ is a selection matrix with all zeros except for the parts of the state $x$ that correspond to an exploration space in the subset of the robot's global position in the world.  } 
        The map $g$ then takes us from state space $\mathcal{X}$ to exploration space $\expSpace$.

        A trajectory $x(t), \forall t \in [t_0, t_f]$ is \emph{ergodic} with respect to a measure $\phi(w) : \expSpace \to \Re^+$ if and only if 
        \begin{equation} \label{eq:erg_def}
            \lim_{t_f\to \infty} \frac{1}{t_f} \int_{t_0}^{t_f} \mu(g(x(t))) dt = \int_\expSpace \phi(w) \mu(w) dw
        \end{equation}
        for all Lebesgue integrable functions, $\mu \in \mathcal{L}^1$~\cite{scott2009capturing}.
        Because we can not run a robot for $t_f\to \infty$, we consider $t_f < \infty$ where trajectories are sub-ergodic. 
        For a finite $t_f$, where $x(t)$ is a deterministic trajectory we define the left-hand side of Eq.~\eqref{eq:erg_def} as the time-averaged trajectory statistics
        \begin{equation}\label{eq:time_avg}
            c(w, x(t)) = \frac{1}{t_f} \int_{t_0}^{t_f} \delta[w - g(x(t))] dt
        \end{equation}
        where $\delta$ is the Dirac delta function (in place of $\mu$) and $w \in \expSpace$ is a point in the exploration space.
        Using Eq.~\eqref{eq:time_avg} as part of an optimization routine is not possible in the current form as the delta function is not differentiable. 
        To define the ergodic metric for optimization, we use spectral methods and construct a metric in the Fourier space~\cite{mathew2011metrics, scott2009capturing,miller2015ergodic}.

        Let us define the $k^\text{th} \in \mathbb{N}^{v}$ cosine Fourier basis function as 
        \begin{equation}
            F_k(w) = \frac{1}{h_k} \prod_{i=0}^{v-1} \cos\left(\frac{w_i k_i \pi}{L_i}\right)
        \end{equation}
        and $h_k$ is a normalizing factor (see \cite{miller2015ergodic,mathew2011metrics}).
        Then, the ergodic metric is defined as 
        \begin{align}\label{eq:ergodic_met}
            &\ergMetr(x(t), \phi) = \sum_{k\in \mathcal{K}^v} \Lambda_k \left( c_k - \phi_k \right)^2 \\
            &= \sum_{k\in  \mathcal{K}^v} \Lambda_k \left( \frac{1}{t_f}\int_{t_0}^{t_f}F_k(g(x(t))) dt - \int_{\expSpace} \phi(w) F_k(w)dw \right)^2 \nonumber
        \end{align}
        where $k\in\mathcal{K}^v\subset \mathbb{N}^v$ is the set of all fundamental frequencies, $c_k$ and $\phi_k$ are the $k^\text{th}$ Fourier decomposition of $c(w, x(t))$ and $\phi(w)$, respectively, and
        $\Lambda_k=(1 + \left\Vert k \right\Vert)^{-\frac{v+1}{2}}$ is a weight coefficient that places higher importance on lower-frequency modes.

        We formulate the ergodic trajectory optimization problem as the following minimization problem over state and control trajectories $x(t), u(t)$:
        
        % \vspace{2mm}
        \noindent
        \underline{Ergodic Trajectory Optimization:}
        \begin{subequations}\label{eq:erg_traj_opt}
            \begin{align}
                &\min_{x(t),u(t)} \Big{\{}\mathcal{E}(x(t), \phi) + \int_{t_0}^{t_f} u(t)^\top \mathbf{R} u(t)dt \Big{\}} \\ 
                &\text{s.t. } \quad
                    \begin{cases}
                        x \in \mathcal{X}, u \in \mathcal{U}, g(x) \in \mathcal{W} \\
                        x(t_0) = \bar{x}_0, x(t_f) = \bar{x}_f\\
                        \dot{x} = f(x, u) \\
                        h_1(x,u) \le 0, h_2(x,u) = 0 \\
                    \end{cases}
            \end{align}
        \end{subequations}
        where $h_1$, $h_2$ are inequality and equality constraints, and $\mathbf{R} \in \mathbb{R}^{m\times m}$ is a diagonal positive-definite matrix that penalizes control. 

    \subsection{Time-Optimal Control Problem Statement}
    
        Given the same robot trajectories $x(t), u(t)$ and dynamics $\dot{x}=f(x,u)$ defined previously, we define the time-optimal control problem as minimizing time $t_f$.
        However, this alone renders an ill-posed problem with a trivial solution $t_f=0$. 
        To circumvent this issue, it is common to include some terminal state condition $x(t_f) = x_f$ which needs to be satisfied. 
        In addition, constraints are commonly included to further restrict the solution space as one can end up with ``infinite'' control input which is not feasible on robotic systems.     
        The formulation of the time-optimal control problem is then
        
        \noindent
        \vspace{2mm}
        \underline{Time-Optimal Control Problem:}
        \begin{subequations}\label{eq:time-opt}
        \begin{align}
            &\min_{x(t),u(t), t_f} \quad t_f \\ 
            &\text{s.t. } \quad
                \begin{cases}
                    x \in \mathcal{X}, u \in \mathcal{U}, t_f > 0 \\
                    x(t_0) = x_0, x(t_f) = x_f,  \\ 
                    \dot{x} = f(x, u) \\
                    h_1(x,u) \le 0, h_2(x,u) = 0
                \end{cases} \label{subeqn:time-opt-constr}
            % \dot{x} = f(x, u), x \in \mathbb{R}^n, u \in \mathbb{R}^m, h(x,u) \le 0 \\
            %     &x(t_0) = \bar{x}_0, x(t_f) = \bar{x}_f, g(x) \in \mathcal{W}
        \end{align}
        \end{subequations}
        where we optimize over $x(t), u(t)$ and $t_f$, $h_1, h_2$ are inequality and equality constraints respectively, and $\bar{x}_f$ is a terminal state. 

        As an aside, it is worth noting that time-optimal control problems are often similar in formulation to time-optimal trajectory problems. 
        The different use cases depend on the time horizon settings. 
        In long-time horizons, it is often preferred to solve for robot trajectories directly and use them to track points~\cite{foehn2021time}. 
        It is possible to find closed-form feedback-control solutions based on the conditions of optimality from Pontryagin's maximum principle~\cite{kopp1962pontryagin}. 
        In this work, we focus on showing that one can prove the conditions of optimality for the time-optimal ergodic control problem and obtain solutions using a direct trajectory optimization method. 
        We leave computing closed-form solutions to future work. 
        
        Numerical solutions to Eq.~\ref{eq:time-opt} can be obtained by discretizing trajectories over $N$ ``knot'' points where a discrete time is calculated as $\Delta t = \frac{t_f}{N}$~\cite{foehn2021time}.
        The continuous-time dynamics of the robot are then transcribed using an integration method (e.g., forward Euler, implicit Euler, Runge-Kutta) where $x_{t+\Delta t} = x_t + \Delta t f(x_t, u_t)$ denotes an explicit Euler method and the subscripts refer to discrete time points. 
        In the following section, we formalize the time-optimal ergodic search problem and derive analytical solutions using conditions of optimality and numerical solutions using direct trajectory transcription.

\section{Time-Optimal Ergodic Search} \label{sec:time_opt_erg}

    In this section, we formulate and pose the time-optimal ergodic search problem. 
    Solutions to the problem are presented in two manners: 1) analytically through conditions of optimality; and 2) numerically through a direct transcription approach. 
    The analytical approach is used to establish conditions of optimality (which can be seen as continuous time analogies of the KKT-conditions~\cite{tabak1971optimal, kuhn1951nonlinear}). 
    We derive these results as purely analytical, with the intent that these conditions are to be used to further analyze the structure of the time-optimal ergodic search problem in future work. 
    The numerical approach provides a direct form of calculating robot trajectories and control solutions for the time-optimal ergodic search problem. 

    \subsection{Problem Formulation}

        Let us consider the same robot with state and control trajectories $x(t), u(t)$ with continuous time (nonlinear) dynamics $\dot{x} = f(x,u)$. 
        In addition, consider the bounded exploration space $\expSpace$ with map $g: \mathcal{X} \to \expSpace$ and information measure $\phi(w)$.
        Our goal is to optimize search time $t_f$ while minimizing the ergodic metric Eq.~\eqref{eq:ergodic_met}.
        Let us first make more explicit the dependence of the ergodic metric on the time $t_f$
        \begin{align}\label{eq:ergodic_met_time}
            &\ergMetr(x(t), \phi, t_f) = \sum_{k\in \mathcal{K}^v} \Lambda_k \left( c_k(x(t), t_f) - \phi_k \right)^2 \\
            &= \sum_{k\in \mathcal{K}^v} \Lambda_k \left( \frac{1}{t_f}\int_{t_0}^{t_f}F_k(g(x(t))) dt - \int_{\expSpace} \phi(w) F_k(w)dw \right)^2, \nonumber
        \end{align}
        where $c_k(x(t), t_f)$ is the term that depends on time. 
        
        According to Eq.~\eqref{eq:erg_def}, a trajectory can become ergodic as $t_f \to \infty$. 
        This makes the problem of time-optimal ergodic search ill-posed as $t_f$ will be significantly large if the ergodic metric is to be minimized. 
        To solve for this, we propose to include the ergodic metric as an inequality constraint $\mathcal{E}(x(t), \phi, t_f) \le \gamma$, where $\gamma \in \mathbb{R}^+$ is an upper bound on ergodicity. 
        As an example of how $\gamma$ can still provide sufficient coverage, consider that the ergodic metric is defined over Fourier modes. 
        Satisfying the ergodic constraint then requires minimizing the distance between the $k^\text{th}$ modes of $c_k$ and $\phi_k$ such that the sum of squares is less than $\gamma$. 
        Because we are working with \emph{spectral} Fourier modes that span the exploration space $\expSpace$, this implies that $\gamma$ imposes a lower bound on ergodic coverage based on the spectral bands with the highest amplitudes.
        Therefore, we can use the ergodic inequality constraint as an additional condition for time optimization so the problem is well-posed. 

        Including a terminal state condition $x(t_f) = x_f$ as a secondary boundary condition, the time-optimal ergodic search problem using~\eqref{eq:time-opt} and~\eqref{eq:erg_traj_opt}is defined as:

        \vspace{2mm}
        \noindent
        \underline{Time-Optimal Ergodic Trajectory Optimization:}
        \begin{subequations}\label{eq:time-opt-erg}
            \begin{align}
                &\min_{x(t),u(t),t_f} \quad t_f \\ 
                &\text{s.t. } \quad
                    \begin{cases}
                        x \in \mathcal{X}, u \in \mathcal{U} \\
                        x(t_0) = x_0 \\
                        \dot{x} = f(x, u) \\
                        x(t_f) = x_f, g(x) \in \mathcal{W} \\ 
                        h_1(x,u) \le 0, h_2(x,u) = 0 \\
                        \mathcal{E}(x(t), \phi, t_f) \leq \gamma, t_f >0
                    \end{cases}
            \end{align}
        \end{subequations}
        where the last set of constraints ensures that solutions are minimizing the ergodic metric up to $\gamma$ and time is always positive. 
        In this problem, $h_1$ and $h_2$ often encode additional control constraints, e.g., so that $u(t)$ is bounded or that $x(t)$ avoids obstacles in the environment.
        The following subsections propose solutions to the time-optimal ergodic search problem.

    \subsection{Indirect Solution via Pontryagin's Maximum Principle}

        We can show that there exist analytical conditions of optimality (as done with the original time-optimal control results) for the time-optimal ergodic search problem in~\eqref{eq:time-opt-erg}.
        To show this, we first express the problem~\eqref{eq:time-opt-erg} without constraints $h_1, h_2$ (these can be later introduced, but for now, we are interested in the simpler problem). 
        In addition, we formulate an objective function using Lagrange multipliers $\lambda(t)$, $\rho_1$, and $\rho_2$:
        \begin{align} \label{eq:bolza_form}
            \mathcal{J}(x(t), u(t), t_f) &= \rho_1\bar{\mathcal{E}}(x(t), t_f)
            + \rho_2^\top (x - \bar{x})\mid_{t_f} \nonumber \\
            & + \int_{t_0}^{t_f} 1 + \lambda^\top\left( f(x,u) - \dot{x}\right) dt
        \end{align}
        where $\bar{\mathcal{E}}(x(t), t_f) = -\log(-\mathcal{E}(x(t), \phi, t_f)+\gamma)$ is a log barrier term that represents the inequality constraint~\cite{potra2000interior}. 
        The representation in Eq.~\eqref{eq:bolza_form} appears to be in a Bolza form where the cost of time $t_f$ is introduced with the added $1$ under the integral, i.e., $\int_{0}^{t_f} 1 dt = t_f$. 
        Typically, it is sufficient to apply the maximum principle to the Bolza form and obtain conditions of optimality. 
        However, note that $\bar{\mathcal{E}}$ requires the full trajectory and not simply the terminal time which makes the problem not in Bolza form.
        As a result, we need to further simplify the problem. 

        To do so, we define an extended ergodic state~\cite{de2016ergodic}: 
        \begin{definition} \label{def:1}\textbf{Extended ergodic state.}
            The ergodic metric~\eqref{eq:ergodic_met} can be equivalently expressed as 
            \begin{align*}
                \mathcal{E}(x(t), \phi, t_f) &= \sum_{k\in \mathcal{K}^v} \Lambda_k \left( c_k(x(t), t_f) - \phi_k \right)^2 \\ 
                &= \frac{1}{t_f^2}\Vert z(t_f) \Vert_\mathbf{\Lambda}^2
            \end{align*}
            where $z(t_f) = [z_0, z_1, \ldots, z_{|\mathcal{K}^v|}]^\top$ is the solution to
            \begin{equation}
                \dot{z}_k = F_k(g(x(t))) - \phi_k
            \end{equation}
            with initial condition $z(t_0) = \mathbf{0}$, and $\mathbf{\Lambda} = \text{diag}(\Lambda)$ is a diagonal matrix consisting of the weights $\Lambda=[\Lambda_0, \ldots, \Lambda_{|\mathcal{K}^v|}]$.
        \end{definition}
        \begin{proof}
            From~\cite{de2016ergodic}, it can be shown that by multiplying time $t$ we can define
            \begin{align} \label{eq:equiv_zk}
                z_k(t) & = c_k(x(t), t) - t\phi_k \nonumber \\
                        & = \int_0^t F_k(g(x(\tau)))d\tau - t\int_\mathcal{W} F_k(w)\phi(w)dw.
            \end{align}
            When $t=t_f$, it can be readily shown that $\sum_{k\in  \mathcal{K}^v} \Lambda_k \left( c_k(x(t), t_f) - \phi_k \right)^2 = \frac{1}{t_f^2}\Vert z(t_f) \Vert_\mathbf{\Lambda}^2$.
            Taking the derivative of Eq.~\eqref{eq:equiv_zk} with respect to time, we get $\dot{z}_k = F_k(g(x(t))) - \phi_k$
            which we define as the governing differential equation for the extended ergodic state.
        \end{proof}
    
        With the extended ergodic state, we are able to extend the dynamics of the system 
        \begin{equation} \label{eq:extended_dynamics}
            \dot{\bar{x}} = \bar{f}(\bar{x}, u) = \begin{bmatrix}f(x, u) \\ \mathbf{F}_k(g(x(t))) - \Phi_k\end{bmatrix}
        \end{equation}
        where $\bar{x} = [x^\top, z^\top]^\top$ is the extended state, $\mathbf{F}_k(w)=[F_0(w), F_1(w), \ldots, F_{| \mathcal{K}^v|}(w)]^\top$ is a vector of all the Fourier basis functions, and $\Phi_k = [\phi_0,\phi_1, \ldots, \phi_{| \mathcal{K}^v|}]^\top$ is a vector of all the Fourier coefficients of $\phi$.
        The objective function can then be written compactly as 
        \begin{align} \label{eq:obj_compact}
            \mathcal{J}(\bar{x}(t), u(t), t_f) &= \rho^\top \psi(\bar{x}, t_f)\mid_{t_f} \nonumber \\
            % + \rho_2^\top (x - \bar{x})\mid_{t_f} \nonumber \\
            & + \int_{t_0}^{t_f} 1 + \lambda^\top\left( \bar{f}(\bar{x},u) - \dot{\bar{x}}\right) dt
        \end{align}
        where $\psi(\bar{x}, t_f) = [-\log(-\frac{1}{t_f^2}\Vert z(t_f) \Vert_\mathbf{\Lambda}^2+\gamma), (x - \bar{x})^\top]^\top$.
        Defining the Hamiltonian of the control system as 
        \begin{equation}\label{eq:ham}
            H(\bar{x},u,\lambda) = 1 + \lambda^\top\bar{f}(\bar{x},u)
        \end{equation}
        we are able to apply the maximum principle and obtain conditions of optimality. 
        \begin{theorem} \label{thrm:1}\textbf{Conditions of (Local) Optimality.}
            For a control system that follows the dynamics $\dot{\bar{x}} = \bar{f}(\bar{x},u)$ (Def.~\ref{def:1}) and Hamiltonian
            \begin{equation*}
                H(\bar{x},u,\lambda) = 1 + \lambda^\top\bar{f}(\bar{x},u),
            \end{equation*}
            the tuple $(\bar{x}(t), u(t), \lambda(t), t_f)$ is a locally optimal solution to the time-optimal ergodic search problem~\eqref{eq:time-opt-erg} over the free time interval $t \in [0, t_f]$ if the following conditions are satisfied:
            \begin{subequations}\label{eq:cond_opt}
                \begin{align}
                    \psi(\bar{x}(t_f), t_f) = 0 \\ 
                    \bar{x}(t_0) = \bar{x}_0 \\ 
                    \dot{\bar{x}} = \frac{\partial H}{\partial \lambda}^\top \\ 
                    \dot{\lambda} = - \frac{\partial H}{\partial \bar{x}}^\top \\ 
                    u^\star = \argmin_{u \in \mathcal{U}} H(\bar{x}^\star, u, \lambda^\star) \\
                    \lambda(t_f) =  \frac{\partial \psi}{\partial \bar{x}}^\top \rho \bigg|_{t_f} \\
                    H(\bar{x}(t_f),u(t_f),\lambda(t_f), t_f) = -\rho^\top \frac{\partial \psi }{\partial t_f} \bigg|_{t_f}
                \end{align}
            \end{subequations}
        \end{theorem}
        \begin{proof}
            See Appendix~\ref{appendix:proof}. 
        \end{proof}
        This theorem provides evidence that time-optimal ergodic solutions (if they exist) can satisfy a set of continuous-time conditions for optimality. 
        In practice, it is possible to use these conditions to generate control solutions to the time-optimal ergodic search problem. 
        However, we found that they do not work for long time horizons due to the numerical instabilities when computing the two-point boundary value problem from equations~\eqref{eq:cond_opt}(c) and (d). 
        Instead, we use an approximate direct optimization method using the KKT conditions over $N$ discrete knot points to solve for time-optimal ergodic trajectories which we describe in the following subsection. 

    \subsection{Direct Solutions via Transcription}
        In this section, we outline a direct transcription method for numerically solving~\eqref{eq:time-opt-erg}. 
        Our approach is similar to that of prior time-optimal planning methods~\cite{foehn2021time, posa2014direct}.
        
        We first begin by defining the continuous-time dynamics $\dot{x} = f(x,u)$ as a discrete-time system over a sequence of $N$ discretized ``knot'' points:
        \begin{equation}
            x_{t+1} = x_t + \Delta t f(x_t, u_t)
        \end{equation}
        where $\Delta t = \frac{t_f}{N}$ and the subscripts define a discrete time point. Note that we depict an Explicit Euler integration scheme, but this is not specific to our method and can be changed to a Runge-Kutta or Implicit integration scheme. 
        Next, we define the optimization variables as $\mathbf{x} = \{ x_0, \ldots, x_N \}$, $\mathbf{u} = \{ u_0, \ldots, u_{N-1} \}$. 
        The ergodic metric in discrete-time becomes 
        \begin{align}
            &\mathcal{E}(x(t), \phi, t_f) \approx \hat{\mathcal{E}}(\mathbf{x}, \phi, t_f)
            = \sum_{k\in  \mathcal{K}^v} \Lambda_k \left( c_k(\mathbf{x}, t_f) - \phi_k \right)^2  \nonumber\\
            &= \sum_{k\in  \mathcal{K}^v} \Lambda_k \left( \frac{1}{t_f}\sum_{t=0}^{N-1}F_k(g(x_t)) \Delta t - \phi_k\right)^2
        \end{align}
        Since we describe the time discretization $\Delta t$ as derived from $t_f$, we can directly write the optimization problem over $\mathbf{x}, \mathbf{u}, t_f$ as a nonlinear program (NLP): 
        
        \vspace{2mm}
        \noindent
        \underline{Direct Time-Optimal Ergodic Trajectory Optimization:}
        \begin{subequations}\label{eq:direct_time-opt-erg}
            \begin{align}
                &\min_{\mathbf{x},\mathbf{u},t_f} \quad t_f \\ 
                &\text{s.t. } \quad
                    \begin{cases}
                        \Delta t = \frac{t_f}{N} \\
                        x_t \in \mathcal{X}, u_t \in \mathcal{U} \\
                        x_0 = x_0, x_{t+1} = x_t + \Delta t f(x_t, u_t) \\
                        x_N = x_f, g(x_t) \in \mathcal{W} \\ 
                        h_1(\mathbf{x},\mathbf{u}) \le 0, h_2(\mathbf{x},\mathbf{u}) = 0 \\
                        \hat{\mathcal{E}}(\mathbf{x}, \phi, t_f) \leq \gamma, t_f >0
                    \end{cases}
            \end{align}
        \end{subequations}
        The optimization in \eqref{eq:direct_time-opt-erg} is solved as a direct-collocation problem, i.e., optimization variables are free where constraints impose physical robot limitations and dynamics. Note that as a result of this implementation, the initial conditions may be chosen arbitrarily and the chosen solver will ``stitch'' together a trajectory based on the relevant constraints. 
        We use a NLP solver (specifically a custom variation of an augmented Lagrangian constrained optimization solver~\cite{potra2000interior, kuhn1951nonlinear}), that directly solves~\eqref{eq:direct_time-opt-erg}. 
        Solutions are verified against conditions (to establish convergence) using the KKT conditions of the NLP problem~\cite{kuhn1951nonlinear}. Note that it is possible to use the optimality conditions in Theorem~\ref{thrm:1}; however, these will only be approximate conditions due to the time discretization.
        Furthermore, the choice of initial condition determines whether the solver converges. Because the ergodic metric is highly nonlinear and non-convex, the dependence of solutions as a function of initial conditions may vary drastically (as shown in~\cite{miller2013trajectory}). We fix the initial trajectory conditions to be a linear interpolation between the initial and final conditions for all examples. 
        In the following section, we demonstrate simulated and experimental results for solutions to~\eqref{eq:direct_time-opt-erg}.

\section{Results} \label{sec:results}

    In this section, we demonstrate simulated and experimental results for time-optimal ergodic search in several scenarios.
    The proposed approach is evaluated as a trajectory optimizer in settings where obstacles in the environment are known and the utility of coverage over specific areas is provided to the robot. \footnote{The construction of the coverage utility $\phi$ is often a function of new measurements collected from executing the time-optimal trajectories and independent of how coverage trajectories are optimized.}
    Specifically, we are interested in scenarios that allow us to investigate aspects of the proposed approach that answer the following questions:
    \begin{enumerate}
        \item[Q1:] Can we generate time-optimal ergodic trajectories, and how do they compare to fixed-time ergodic trajectories?
        \item[Q2:] Do we retain the ability to bias the search with a non-uniform $\phi$?
        \item[Q3:] What influence does the ergodic upper bound $\gamma$ have on optimized time (and control)?
        \item[Q4:] How much do trajectory solutions change with changes in discretizing knot points $N$ and how do initial conditions affect generated solutions?
        \item[Q5:] Is it possible to include constraints in the problem formulation~\eqref{eq:direct_time-opt-erg} to explore in more realistic scenarios, e.g., a cluttered environment?
        \item[Q6:] And last, do optimized trajectories transfer to real-world drone applications for search in a cluttered environment?
    \end{enumerate}
    We organize the results section such that each of these questions are answered sequentially (with A\#:) and through the figures starting from Fig. 2. 
    Implementation details are provided in the text and in Appendix~\ref{appendix:implementation}. 

        \begin{figure}
            \centering
            \includegraphics[width=\linewidth]{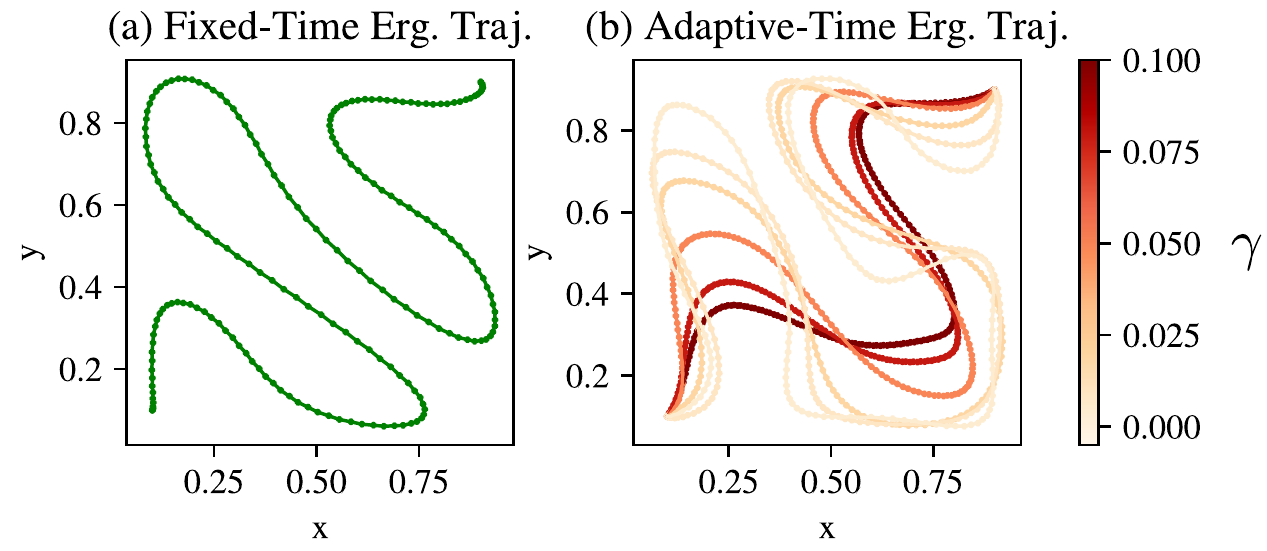}
            \caption{\textbf{Ergodic Trajectory Sensitivity Analysis.} Trajectory solutions found using a uniform distribution $\phi$. (a) A fixed time $t_f=10$s ergodic trajectory solution with resulting ergodicity $\mathcal{E}=0.007$. (b) Solutions to time-optimal ergodic trajectories Eq.(\ref{eq:time-opt-erg}) with varying $\gamma$. Time solutions range from $5$-$10$s depending on $\gamma$ with equivalent coverage to the fixed-time ergodic trajectory being obtained with $<6$s. }
            \label{fig:comparison}
        \end{figure}
    \subsection{Simulated Results}

        \noindent
        \textbf{A1: Comparison to Fixed-Time Ergodic Search.} 
        Our first result compares the proposed direct solution~\eqref{eq:direct_time-opt-erg} to the original ergodic trajectory optimization problem~\eqref{eq:erg_traj_opt}.
        For this result, we use a 2-D double integrator (or point-mass) dynamical system as our simulated robotic system whose goal is to uniformly explore a bounded exploration space $\expSpace$. 
        We specify $\phi(w)$ as a uniform distribution and use initial solver parameters $t_{f,\text{init}}=10s$, $N=200$, and control penalization matrix $\mathbf{R} = 0$ for \eqref{eq:erg_traj_opt}.
        
        Fig.~\ref{fig:comparison} is a side-by-side comparison against the fixed-time ergodic trajectory solved using~\eqref{eq:erg_traj_opt} and the proposed time-optimal approach~\eqref{eq:direct_time-opt-erg}. 
        Note that with the fixed-time ergodic search method, planned trajectories are limited in how they explore an area based on the initial planning time $t_f$. 
        In contrast, optimizing time alongside ergodic trajectories can be seen to provide a range of granularity in how much the trajectory uniformly covers the space. 
        Specifically, as one decreases the value of $\gamma$, optimized trajectories focus more on being ergodic and have less emphasis on optimizing time. 
        With larger values of $\gamma$, time is prioritized with less emphasis on ergodicity so long as trajectories satisfy the ergodic upper bound.

               \begin{figure}
            \centering
            \includegraphics[width=\linewidth]{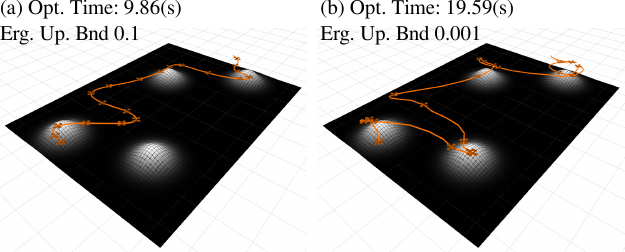}
            \caption{\textbf{Time Opt. Ergodic Search Over Info. Distribution.} Using the information distribution $\phi$, it is possible to solve biased time-optimal ergodic search problems. Shown above are time-optimal trajectory solutions for (a) $\gamma=0.1$, and (b) $\gamma=0.001$. To compensate for the increased coverage requirements imposed by $\phi$, ergodic trajectories are solved to optimize time spent over areas of high information (illustrated as the lighter regions). With tighter requirements on ergodicity with $\gamma=0.001$, it can be seen that the optimized trajectories spend more time proportionally in the areas of high information.   }
            \label{fig:bias_search}
        \end{figure}

        \begin{figure}
            \centering
            \includegraphics[width=0.8\linewidth]{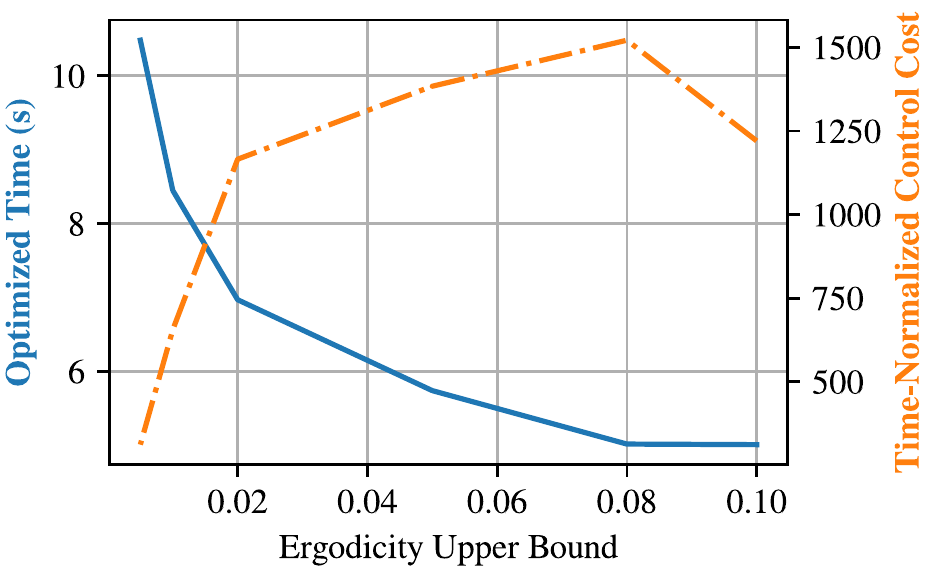}
            \caption{\textbf{Time-Opt. Variations to Ergodicity Upper Bound.} With an increased ergodicity upper bound $\gamma$ (i.e., less emphasis on coverage), time is prioritized. Aligned with ergodic theory~\cite{mathew2011metrics}, as $\gamma \to 0$, the optimized time asymptotically approaches infinity which is required for complete ergodic coverage. Interestingly, time-normalized control cost $\frac{1}{t_f}\int_t \Vert u(t) \Vert dt$ curve implies increased actuation when trading off between coverage requirements and minimizing time. Trajectories are solved using uniform coverage distribution $\phi$.  }
            \label{fig:erg_bnd_ablation}
        \end{figure}

        \vspace{1.5mm}
        \noindent
        \textbf{A2: Biasing Time-Optimal Search with Inf. Distribution.} We can further investigate the efficacy of time-optimal search with respect to a non-uniform $\phi$. 
        This is of interest because time optimization may impact how much time trajectories spend on high-information areas. 
        In addition, resulting ergodic trajectories can overlook important high-information areas that are critical for search. 

        To test this, we define a non-uniform distribution for $\phi$ as illustrated in Fig.~\ref{fig:bias_search}. 
        The distribution consists of four identical Gaussian peaks placed over the exploration space $\expSpace$ (see Appendix~\ref{appendix:implementation} for more detail).
        We use the same 2-D point mass dynamics but test only $\gamma=0.1$ and $\gamma=0.001$ which indicate a coarse search with an emphasis on optimizing time and a finer search with less emphasis on time respectively. 
        Trajectories illustrated in Fig.~\ref{fig:bias_search} show that even with a high $\gamma$, the generated trajectory still visits each Gaussian peak.
        However, the trajectory does not spend too much time in the area and brushes past the first peak. 
        This is the result of the ergodic inequality constraint and the balance of time versus coverage. 
        This can be seen in the difference in elapsed optimal times of $9.86s$ and $19.59$ seconds respectively (almost $2\times$ increase in time in response to 2 orders of magnitude of reduction on $\gamma$. 

       \begin{table}[h!]
            \vspace{-6mm}
            \centering  
            \caption{Optimized Time Parameter Sensitivity}
            \begin{tabular}{|l|cc|}
            \hline
            \multirow{2}{*}{Parameter} & \multicolumn{2}{c|}{Optimized Time $\gamma=0.05$} \\ \cline{2-3} 
                                       & \multicolumn{1}{c|}{Mean}     & Std. Deviation    \\ \hline
            $t_{f,\text{init}}$ (4-8)s & \multicolumn{1}{c|}{4.97s}    & $\pm$ 0.23s           \\ \hline
            Knot Points $N$ (50-600)   & \multicolumn{1}{c|}{5.45s}    & $\pm$ 0.39s            \\ \hline
            \end{tabular}
            \label{tab:ablation}
        \end{table}

        \vspace{1.5mm}
        \noindent
        \textbf{A3,4: Parameter Ablation Studies.} Given the drastic change in performance of optimized time against the change of $\gamma$ for optimized ergodic trajectories, we investigate parameter sensitivity through two ablation studies.
        The first study looks at the change of optimal time against changes in $\gamma$. 
        The second study investigates the effect of initial conditions and the influence of the time discretization introduced by the knot points $N$.
        For both studies, the same 2-D point-mass system is used and $\phi$ defines a uniform distribution over the bounded exploration space. 

        \begin{figure}[h!]
            \centering
            \includegraphics[width=\linewidth]{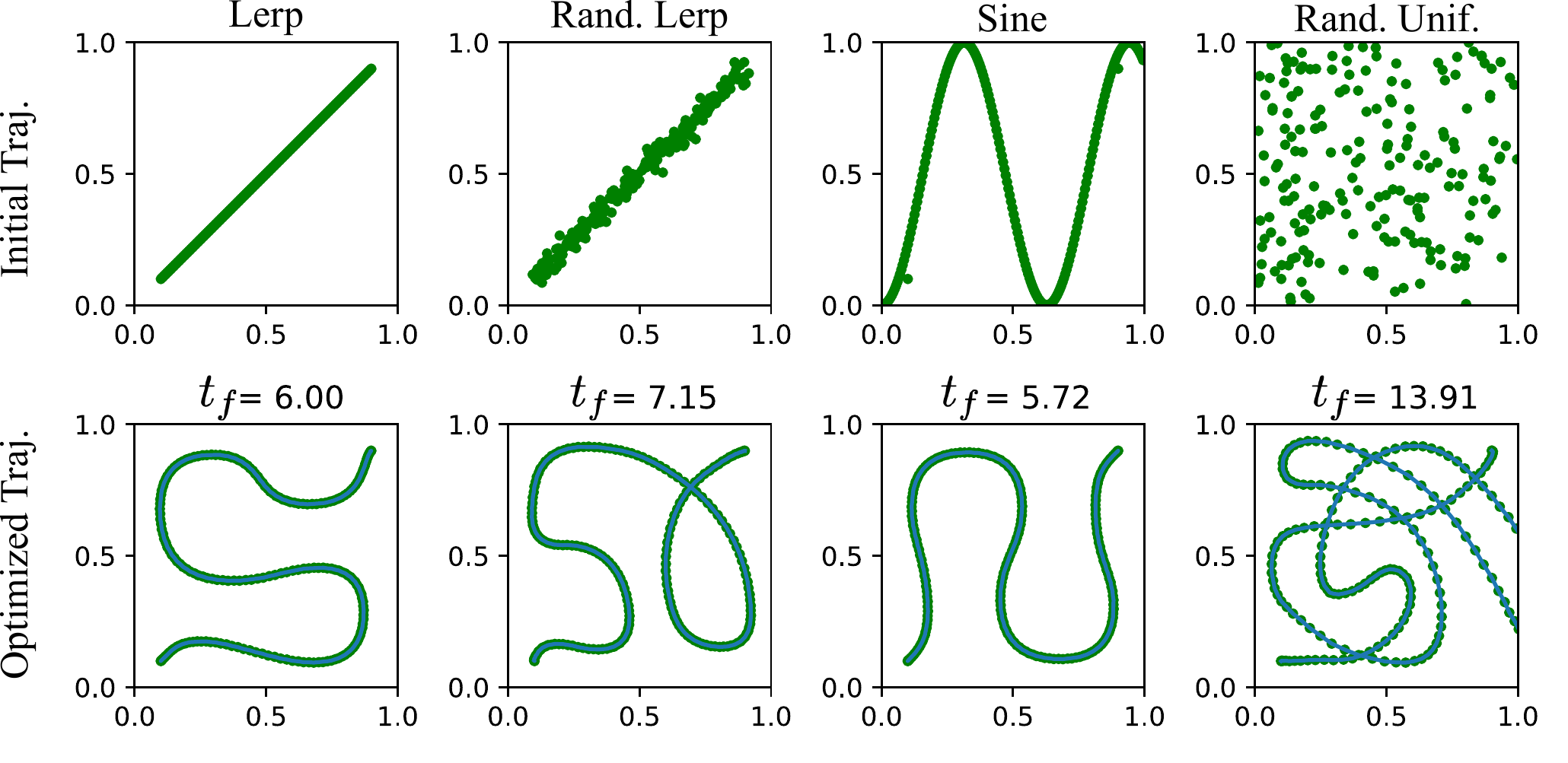}
            \caption{\textbf{Initial Trajectory Ablation.} Here, we show the dependence of time-optimal ergodic solutions on the initial trajectory that was provided to the solver. From left to right we show optimized solutions for linear interpolation (Lerp) from initial to final condition, the linear interpolated trajectory with added normally distributed zero-mean noise with standard deviation $0.02$, a sinusoidal initial condition, and a uniformly random initial condition. Each solution satisfies an ergodicity of $\mathcal{E}=0.01$ with initial time $t_f=10$ and control knots $N=200$. We find that the deviation in final optimized time depends on path smoothness. Non-smooth initial trajectories tend to fall into equally ergodic local minima, causing worse optimized time.  }
            \label{fig:erg_init_ablation}
        \end{figure}
                
        Figure~\ref{fig:erg_bnd_ablation} illustrates the results of the first study. 
        As $\gamma\to0$, the ergodic inequality constraint becomes more of an equality constraint. 
        This is due to the ergodic metric being lower-bounded by $0$ by definition. 
        As a result, the smaller $\gamma$ becomes, the more the optimized time tends towards $\infty$!\footnote{There is a point where the solver does not provide solutions as $t_f$ is required to be significantly larger and is dependent on the initial trajectory condition and the dynamic constraints.}
        This exactly corresponds to the statement of ergodicity~\eqref{eq:erg_def} that defines a trajectory as being ergodic only at the limit of $t_f\to \infty$.
        What is interesting in the control trade-off shown in Fig.~\ref{fig:erg_bnd_ablation}. 
        Plotted is the time-normalized control $\frac{1}{t_f}\int_{0}^{t_f} \Vert u(t) \Vert dt$
        where the value of $u(t)$ is bounded by $u_\text{max}$ through added constraints. 
        In low $\gamma$ values (large optimized $t_f$), less control effort is needed to be ergodic. 
        We suspect this is due to the robot leveraging its dynamics to slowly navigate an area without the need to change direction abruptly. 
        On the other spectrum of $\gamma$, control actuation begins to fall as time is prioritized. 
        This is due to $\gamma$ reaching an upper bound on the ergodic metric (as the metric is composed of only cosine functions). 
        Therefore, there is less emphasis to be ergodic and more emphasis to optimize time (less direct changes in actuation). 
        The balance between optimizing time and being ergodic is then shown to require more actuation.

        We further investigate the dependence of the optimized time against the initial condition as the discretizing knot points $N$. 
        Experimental runs are done using the same point-mass dynamics in the bounded environment with a uniform distribution as $\phi$ with $\gamma=0.05$. 
        We vary the initial time $t_{f,\text{init}}$ between $4-8s$ with $1s$ intervals and $N=200$ which we found to be a range where the solver would provide solutions within acceptable tolerances. 
        In addition, we varied the number of knot points between $50-600$ with a resolution of $100$ after $50$ with $t_{f,\text{init}}=10$.
        In Table~\ref{tab:ablation} the optimized time solutions along with the standard deviation are provided. 
        Optimized time solutions tended to stay near $5s$ with a standard deviation of $\pm 0.23s$. We found that knot points having more of an effect on optimized time with $\pm 0.39s$ of standard deviation. 
        The difference in solution is anticipated as ergodic trajectories parameterize the time-average distribution~\eqref{eq:time_avg} which can have an infinite number of solutions that yield the same distribution. 
        As a result, deviations in trajectories may satisfy the ergodic inequality, but provide different time-optimal solutions $t_f$.

        \begin{figure}
            \centering
            \includegraphics[width=\linewidth]{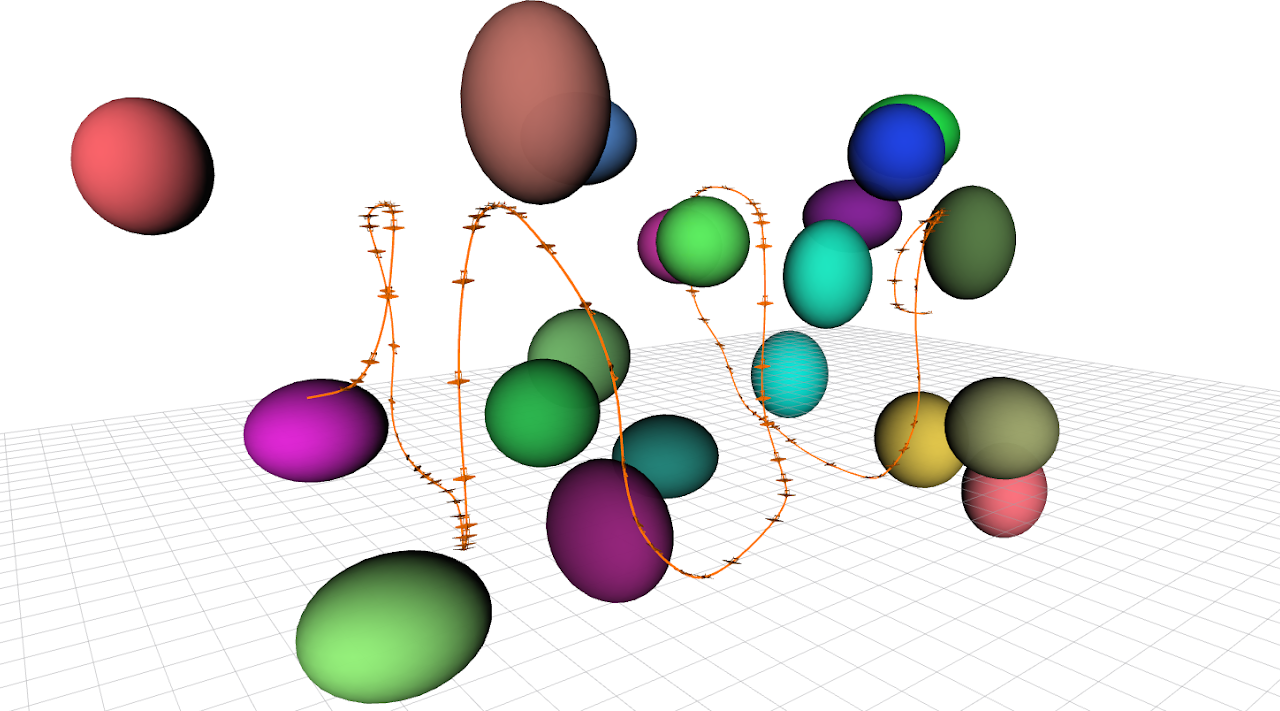}
            \caption{\textbf{Time-Opt. Uniform Ergodic Search with Nonlinear Aircraft Dynamics.} The proposed optimization method is capable of incorporating nonlinear dynamics in $\mathcal{W} \subset \mathbb{R}^3$ with safety-based collision avoidance constraints~\cite{lerch2022safety}. Time-optimal coverage trajectories can be computed ahead of time and executed on the physical system. Collected information can be used to update and bias search.}
            \label{fig:3d_exploration}
        \end{figure}

   \begin{figure*}[ht!]
        \centering
        \includegraphics[width=\textwidth]{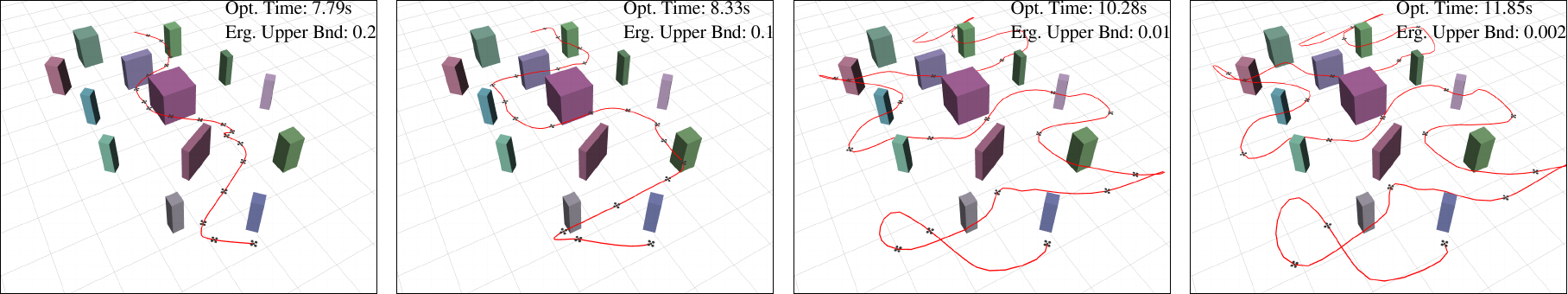}
        \caption{\textbf{Time Optimal Ergodic Trajectory Evolution in Cluttered Environment.} Here, we illustrate the evolution of ergodic trajectories for uniformly exploring the cluttered environment with varying upper bound $\gamma$ on ergodicity. From left to right the minimum required ergodicity, defined by the upper bound $\gamma$ on the optimization problem~\eqref{eq:time-opt-erg}, generates more coverage over the space as $\gamma \to 0$. As the trajectories are required to be more ergodic and cover more of the search area, the trajectory time is automatically increased.  }
        \label{fig:sim_cluttered_search}
    \end{figure*}

    \begin{figure*}[h]
        \centering
        \includegraphics[width=\linewidth]{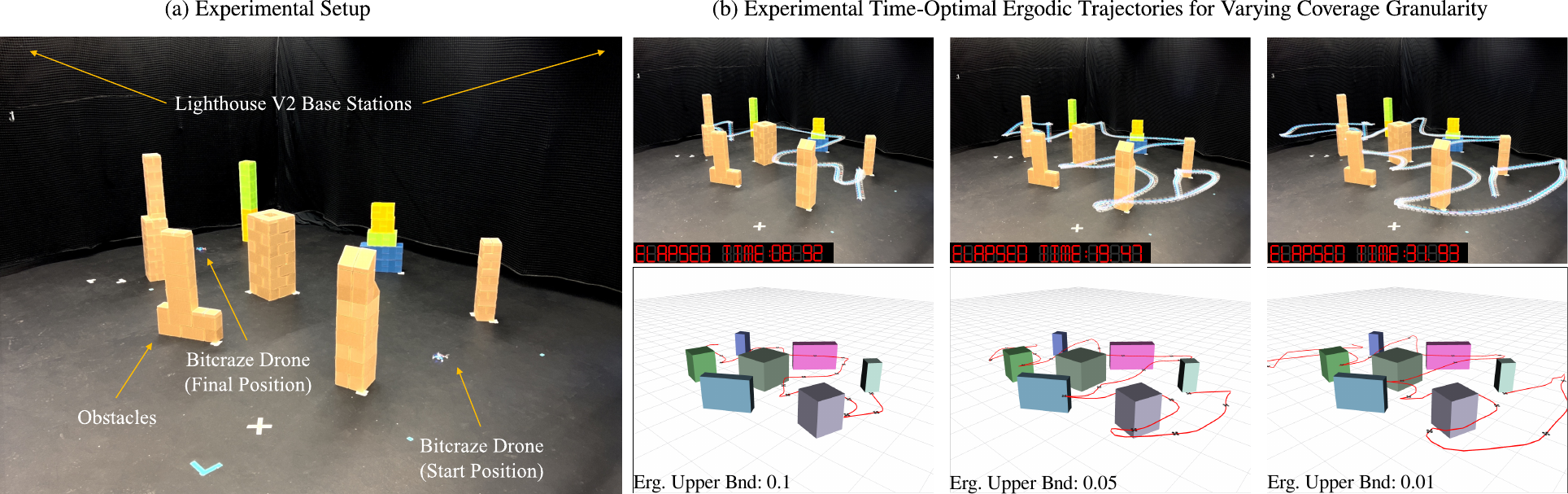}
        \caption{\textbf{Experimental Time-Optimal Ergodic Search Results.} Illustrated are three experimental runs of a drone uniformly exploring an area with obstacles in optimal time. (a) Our experimental setup consists of two Lighthouse position trackers for the drone, a Crazyflie drone, and several obstacles with known positions during trajectory optimization time. (b) Trajectory solutions are found using a uniform distribution $\phi$ and safety-based collision avoidance constraints~\cite{ames2019control}. (Bottom) Solved time-optimal trajectories for varying $\gamma$. (Top) Executed trajectories on the drone with the resulting elapsed time. All drone execution times are within $1.5s$ of the optimized time.  Please see supplementary multimedia for video demonstrations. }
        \label{fig:exp_results}
    \end{figure*}
    
        Last, we studied the dependence of solutions as a function of the initial trajectory that was provided to the solver. 
        The initial trajectory is varied based on common choices (e.g., randomly added noise and sinusoidal paths). 
        Illustrated in Fig.~\ref{fig:erg_init_ablation} is the resulting optimized time-optimal ergodic trajectory subject to the initial trajectory condition. 
        We find that the more regular and smooth the initial trajectory is, the more well-behaved and consistent the optimized solution. 
        Non-smooth initial trajectories provided high variability in the solution. We believe this is caused by the non-linearity of the ergodic metric and that there exist infinitely many trajectory solutions that satisfy the same ergodicity (all solutions maintain an ergodicity of $\mathcal{E}=0.01$) as shown previously in~\cite{miller2013trajectory}.

    \subsection{Time-Optimal Ergodic Search in a Cluttered Environment} \label{sec:results:clutter}

        In this subsection, we investigate more realistic settings for which to use the proposed time-optimal ergodic search. 
        Specifically, we consider the case of time-optimal exploration in a cluttered environment where the goal is for a robot to navigate around obstacles in the environment and cover the whole area. 
        We first demonstrate the results in simulation and show that it is possible to add in safety-based collision constraints~\cite{agrawal2017discrete} without impeding the coverage performance. 
        Then we execute the time-optimal trajectories on a drone.

        \vspace{1.5mm}
        \noindent
        \textbf{A5: Integrating Safety-Based Collision Constraints.} To successfully navigate and explore in a cluttered environment in optimal time, safety-based constraints are required. 
        We introduce safety here through control-barrier functions (CBFs)~\cite{agrawal2017discrete,ames2019control}. 
        CBFs provide an inequality constraint that, when satisfied, guarantees state trajectories remain within a predefined safe set of states. For more information, please see Appendix~\ref{appendix:implementation}.
        The constraints are integrated such that each CBF is centered around an object scattered in the environment (see~\cite{lerch2022safety}).  
        In this example, we assume that we know the location of each obstacle and the goal is to uniformly explore the cluttered area. 
        We use a constrained 2-D single integrator system (kinematic system) as it closely matches the Crazyflie 2.0 drone movements which have limits on how fast they can fly. 
        The CBF constraints are integrated through $h_1$ found in~\eqref{eq:direct_time-opt-erg} where more details can be found in Appendix~\ref{appendix:implementation}.
        
        Results for time-optimal trajectories are illustrated in Fig.~\ref{fig:sim_cluttered_search} for $\gamma = 0.2, 0.1, 0.01, 0.002$. 
        As the value of $\gamma$ decreases, the optimized time is increased.
        This can be seen in Fig.~\ref{fig:sim_cluttered_search} where the trajectory becomes more ergodic and explores each area between the obstacles (taking more time to search carefully as $\gamma$ decreses). 
        The CBF constraints prevent the trajectory from getting too close to any obstacle while allowing the solver to reach the required value of ergodicity. 

        Additionally, we evaluate the proposed method on a nonlinear aircraft dynamics model in a cluttered environment (see Fig.~\ref{fig:3d_exploration} and Appendix~\ref{appendix:implementation} for more detail).
        In this example, the coverage problem is defined in $\mathcal{W} \subset \mathbb{R}^3$ with uniformly random ellipsoids distributed over the search space.
        We find the proposed solver is capable of providing uniform coverage trajectories subject to the nonlinear dynamics constraints. 
        Note that with the dimensionality increasing, so will the computational complexity as described in prior work~\cite{abraham2021ergodic, abraham2018decentralized, shetty2021ergodic}.
        
        \vspace{1.5mm}
        \noindent
        \textbf{A6: Time-Opt. Ergodic Search in Clutter.} Using the Crazyflie 2.0 drone, we demonstrate the ability of a drone to execute time-optimal ergodic search trajectories in a cluttered environment.
        As shown in Fig.~\ref{fig:exp_results} (a), the experimental setup is created so we can track the position of the drone using two Lighthouse trackers. 
        The drone is tasked to explore the environment uniformly, starting at the initial position and ending at the final position in optimal time. 
        As in simulations, the obstacle position and shape are assumed known and a respective CBF safety constraint is used for each obstacle in the environment. 
        Velocity constraints are imposed through the problem constraints specified in~\eqref{eq:direct_time-opt-erg}.

        We test the drone's ability to execute the time-optimal ergodic trajectories in three levels: 1) fast search with $\gamma=0.1$, 2) balanced search with $\gamma=0.05$, and 3) long search with $\gamma=0.01$. 
        The rendered trajectories (bottom) closely match the drone's trajectory (top) in Fig.~\ref{fig:exp_results} (b). 
        The elapsed time is shown on the top figures. 
        Note that the drone execution time and the optimized simulation time are within $1.5s$ which demonstrates good tracking and that the constraints are approximating the drone's behavior closely. 
        Future work will consider real-time control implementation of time-optimal ergodic search where obstacles in the environment are unknown.

\section{Conclusion} 
\label{sec:conclusion}

    In conclusion, we demonstrated a novel time-optimal ergodic trajectory method for synthesizing time-optimal autonomous search and exploration trajectories. 
    We posed the problem of time-optimal ergodic search from the perspective of time-optimal control. 
    Analytical conditions of optimality were proven through a Bolza formulation of the problem and using Pontryagin's maximum principle. 
    A solution based on direct numerical optimization was presented and analyzed for producing time-optimal ergodic trajectories. 
    We show that it is possible to balance time against the granularity of search in several different scenarios that include search in a cluttered environment. 
    The proposed optimization was shown to handle additional constraints without loss of coverage performance. 
    Last, we demonstrated minimum-time search and exploration trajectories in a cluttered environment on a physical drone.

    Future work will consider real-time optimization routines for online planning and control. 
    A limitation of the proposed work is that the solutions are not globally optimal, but  locally optimal solutions. 
    This is due to the ergodic metric being highly nonlinear and non-convex. 
    Interestingly, many trajectory solutions can satisfy the same ergodic metric yielding the same value. 
    Future work will explore the conditions of optimality and the class of trajectories that are considered equivalently optimal. 
    Furthermore, future directions will include environmental uncertainty that has the potential to be integrated into the search and exploration approach.

% \input{sections/citations}
% % \input{sections/acknowledgments}

\bibliography{references}

\begin{thebibliography}{10}
\providecommand{\url}[1]{#1}
\csname url@samestyle\endcsname
\providecommand{\newblock}{\relax}
\providecommand{\bibinfo}[2]{#2}
\providecommand{\BIBentrySTDinterwordspacing}{\spaceskip=0pt\relax}
\providecommand{\BIBentryALTinterwordstretchfactor}{4}
\providecommand{\BIBentryALTinterwordspacing}{\spaceskip=\fontdimen2\font plus
\BIBentryALTinterwordstretchfactor\fontdimen3\font minus
  \fontdimen4\font\relax}
\providecommand{\BIBforeignlanguage}[2]{{%
\expandafter\ifx\csname l@#1\endcsname\relax
\typeout{** WARNING: IEEEtran.bst: No hyphenation pattern has been}%
\typeout{** loaded for the language `#1'. Using the pattern for}%
\typeout{** the default language instead.}%
\else
\language=\csname l@#1\endcsname
\fi
#2}}
\providecommand{\BIBdecl}{\relax}
\BIBdecl

\bibitem{adams2007search}
A.~L. Adams, T.~A. Schmidt, C.~D. Newgard, C.~S. Federiuk, M.~Christie,
  S.~Scorvo, and M.~DeFreest, ``Search is a time-critical event: when search
  and rescue missions may become futile,'' \emph{Wilderness \& Environmental
  Medicine}, vol.~18, no.~2, pp. 95--101, 2007.

\bibitem{mayer2019drones}
S.~Mayer, L.~Lischke, and P.~W. Wo{\'z}niak, ``Drones for search and rescue,''
  in \emph{1st International Workshop on Human-Drone Interaction}, 2019.

\bibitem{prabhakar2020ergodic}
A.~Prabhakar, I.~Abraham, A.~Taylor, M.~Schlafly, K.~Popovic, G.~Diniz,
  B.~Teich, B.~Simidchieva, S.~Clark, and T.~Murphey, ``{Ergodic Specifications
  for Flexible Swarm Control: From User Commands to Persistent Adaptation},''
  July 2020.

\bibitem{Chen-RSS-22}
W.~Chen, R.~Khardon, and L.~Liu, ``{AK: Attentive Kernel for Information
  Gathering},'' in \emph{Proceedings of Robotics: Science and Systems}, New
  York City, NY, USA, June 2022.

\bibitem{tranzatto2022cerberus}
M.~Tranzatto, T.~Miki, M.~Dharmadhikari, L.~Bernreiter, M.~Kulkarni,
  F.~Mascarich, O.~Andersson, S.~Khattak, M.~Hutter, R.~Siegwart \emph{et~al.},
  ``Cerberus in the darpa subterranean challenge,'' \emph{Science Robotics},
  vol.~7, no.~66, p. eabp9742, 2022.

\bibitem{galceran2013survey}
E.~Galceran and M.~Carreras, ``A survey on coverage path planning for
  robotics,'' \emph{Robotics and Autonomous systems}, vol.~61, no.~12, pp.
  1258--1276, 2013.

\bibitem{dai2018quality}
R.~Dai, S.~Fotedar, M.~Radmanesh, and M.~Kumar, ``Quality-aware uav coverage
  and path planning in geometrically complex environments,'' \emph{Ad Hoc
  Networks}, vol.~73, pp. 95--105, 2018.

\bibitem{zelinsky1993planning}
A.~Zelinsky, R.~A. Jarvis, J.~Byrne, S.~Yuta \emph{et~al.}, ``Planning paths of
  complete coverage of an unstructured environment by a mobile robot,'' in
  \emph{Proceedings of international conference on advanced robotics}, vol.~13,
  1993, pp. 533--538.

\bibitem{fazli2010complete}
P.~Fazli, A.~Davoodi, P.~Pasquier, and A.~K. Mackworth, ``Complete and robust
  cooperative robot area coverage with limited range,'' in \emph{2010 IEEE/RSJ
  International Conference on Intelligent Robots and Systems}.\hskip 1em plus
  0.5em minus 0.4em\relax IEEE, 2010, pp. 5577--5582.

\bibitem{siligardi2019robust}
L.~Siligardi, J.~Panerati, M.~Kaufmann, M.~Minelli, C.~Ghedini, G.~Beltrame,
  and L.~Sabattini, ``Robust area coverage with connectivity maintenance,'' in
  \emph{2019 International Conference on Robotics and Automation (ICRA)}.\hskip
  1em plus 0.5em minus 0.4em\relax IEEE, 2019, pp. 2202--2208.

\bibitem{pratissoli2022coverage}
F.~Pratissoli, B.~Capelli, and L.~Sabattini, ``On coverage control for limited
  range multi-robot systems,'' in \emph{2022 IEEE/RSJ International Conference
  on Intelligent Robots and Systems (IROS)}.\hskip 1em plus 0.5em minus
  0.4em\relax IEEE, 2022, pp. 9957--9963.

\bibitem{breitenmoser2010voronoi}
A.~Breitenmoser, M.~Schwager, J.-C. Metzger, R.~Siegwart, and D.~Rus, ``Voronoi
  coverage of non-convex environments with a group of networked robots,'' in
  \emph{2010 IEEE international conference on robotics and automation}.\hskip
  1em plus 0.5em minus 0.4em\relax IEEE, 2010, pp. 4982--4989.

\bibitem{9982287}
A.~Bouman, J.~Ott, S.-K. Kim, K.~Chen, M.~J. Kochenderfer, B.~Lopez, A.-a.
  Agha-mohammadi, and J.~Burdick, ``Adaptive coverage path planning for
  efficient exploration of unknown environments,'' in \emph{2022 IEEE/RSJ
  International Conference on Intelligent Robots and Systems (IROS)}, 2022, pp.
  11\,916--11\,923.

\bibitem{nenchev2013towards}
V.~Nenchev and J.~Raisch, ``Towards time-optimal exploration and control by an
  autonomous robot,'' in \emph{21st Mediterranean Conference on Control and
  Automation}.\hskip 1em plus 0.5em minus 0.4em\relax IEEE, 2013, pp.
  1236--1241.

\bibitem{klesh2008real}
A.~Klesh, A.~Girard, and P.~Kabamba, ``Real-time path planning for time-optimal
  exploration,'' in \emph{AIAA Guidance, Navigation and Control Conference and
  Exhibit}, 2008, p. 6982.

\bibitem{mathew2011metrics}
G.~Mathew and I.~Mezi{\'c}, ``Metrics for ergodicity and design of ergodic
  dynamics for multi-agent systems,'' \emph{Physica D: Nonlinear Phenomena},
  vol. 240, no. 4-5, pp. 432--442, 2011.

\bibitem{miller2013trajectory}
L.~M. Miller and T.~D. Murphey, ``Trajectory optimization for continuous
  ergodic exploration,'' in \emph{2013 American Control Conference}.\hskip 1em
  plus 0.5em minus 0.4em\relax IEEE, 2013, pp. 4196--4201.

\bibitem{coffin2022multi}
H.~Coffin, I.~Abraham, G.~Sartoretti, T.~Dillstrom, and H.~Choset,
  ``Multi-agent dynamic ergodic search with low-information sensors,'' in
  \emph{2022 International Conference on Robotics and Automation (ICRA)}.\hskip
  1em plus 0.5em minus 0.4em\relax IEEE, 2022, pp. 11\,480--11\,486.

\bibitem{abraham2021ergodic}
I.~Abraham, A.~Prabhakar, and T.~D. Murphey, ``An ergodic measure for active
  learning from equilibrium,'' \emph{IEEE Transactions on Automation Science
  and Engineering}, vol.~18, no.~3, pp. 917--931, 2021.

\bibitem{lerch2022safety}
C.~Lerch, D.~Dong, and I.~Abraham, ``Safety-critical ergodic exploration in
  cluttered environments via control barrier functions,'' in
  \emph{International Conference on Robotics and Automation (ICRA)}, 2023.

\bibitem{abraham2018decentralized}
I.~Abraham and T.~D. Murphey, ``Decentralized ergodic control:
  distribution-driven sensing and exploration for multiagent systems,''
  \emph{IEEE Robotics and Automation Letters}, vol.~3, no.~4, pp. 2987--2994,
  2018.

\bibitem{miller2015ergodic}
L.~M. Miller, Y.~Silverman, M.~A. MacIver, and T.~D. Murphey, ``Ergodic
  exploration of distributed information,'' \emph{IEEE Transactions on
  Robotics}, vol.~32, no.~1, pp. 36--52, 2015.

\bibitem{scott2009capturing}
S.~E. Scott, T.~C. Redd, L.~Kuznetsov, I.~Mezi{\'c}, and C.~K. Jones,
  ``Capturing deviation from ergodicity at different scales,'' \emph{Physica D:
  Nonlinear Phenomena}, vol. 238, no.~16, pp. 1668--1679, 2009.

\bibitem{patel2021multi}
S.~Patel, S.~Hariharan, P.~Dhulipala, M.~C. Lin, D.~Manocha, H.~Xu, and
  M.~Otte, ``Multi-agent ergodic coverage in urban environments,'' in
  \emph{2021 IEEE International Conference on Robotics and Automation
  (ICRA)}.\hskip 1em plus 0.5em minus 0.4em\relax IEEE, 2021, pp. 8764--8771.

\bibitem{Abraham-RSS-18}
I.~Abraham, A.~Mavrommati, and T.~Murphey, ``Data-driven measurement models for
  active localization in sparse environments,'' in \emph{Proceedings of
  Robotics: Science and Systems}, Pittsburgh, Pennsylvania, June 2018.

\bibitem{kopp1962pontryagin}
R.~E. Kopp, ``Pontryagin maximum principle,'' in \emph{Mathematics in Science
  and Engineering}.\hskip 1em plus 0.5em minus 0.4em\relax Elsevier, 1962,
  vol.~5, pp. 255--279.

\bibitem{foehn2021time}
P.~Foehn, A.~Romero, and D.~Scaramuzza, ``Time-optimal planning for quadrotor
  waypoint flight,'' \emph{Science Robotics}, vol.~6, no.~56, p. eabh1221,
  2021.

\bibitem{dal2019comparison}
N.~Dal~Bianco, E.~Bertolazzi, F.~Biral, and M.~Massaro, ``Comparison of direct
  and indirect methods for minimum lap time optimal control problems,''
  \emph{Vehicle System Dynamics}, vol.~57, no.~5, pp. 665--696, 2019.

\bibitem{ames2019control}
A.~D. Ames, S.~Coogan, M.~Egerstedt, G.~Notomista, K.~Sreenath, and P.~Tabuada,
  ``Control barrier functions: Theory and applications,'' in \emph{2019 18th
  European control conference (ECC)}.\hskip 1em plus 0.5em minus 0.4em\relax
  IEEE, 2019, pp. 3420--3431.

\bibitem{agrawal2017discrete}
A.~Agrawal and K.~Sreenath, ``Discrete control barrier functions for
  safety-critical control of discrete systems with application to bipedal robot
  navigation.'' in \emph{Robotics: Science and Systems}, vol.~13.\hskip 1em
  plus 0.5em minus 0.4em\relax Cambridge, MA, USA, 2017.

\bibitem{choset2001coverage}
H.~Choset, ``Coverage for robotics--a survey of recent results,'' \emph{Annals
  of mathematics and artificial intelligence}, vol.~31, pp. 113--126, 2001.

\bibitem{araujo2013multiple}
J.~Araujo, P.~Sujit, and J.~B. Sousa, ``Multiple uav area decomposition and
  coverage,'' in \emph{2013 IEEE symposium on computational intelligence for
  security and defense applications (CISDA)}.\hskip 1em plus 0.5em minus
  0.4em\relax IEEE, 2013, pp. 30--37.

\bibitem{bahnemann2021revisiting}
R.~B{\"a}hnemann, N.~Lawrance, J.~J. Chung, M.~Pantic, R.~Siegwart, and
  J.~Nieto, ``Revisiting boustrophedon coverage path planning as a generalized
  traveling salesman problem,'' in \emph{Field and Service Robotics: Results of
  the 12th International Conference}.\hskip 1em plus 0.5em minus 0.4em\relax
  Springer, 2021, pp. 277--290.

\bibitem{cabreira2019survey}
T.~M. Cabreira, L.~B. Brisolara, and F.~J. Paulo~R, ``Survey on coverage path
  planning with unmanned aerial vehicles,'' \emph{Drones}, vol.~3, no.~1, p.~4,
  2019.

\bibitem{ulrich1997autonomous}
I.~Ulrich, F.~Mondada, and J.-D. Nicoud, ``Autonomous vacuum cleaner,''
  \emph{Robotics and autonomous systems}, vol.~19, no. 3-4, pp. 233--245, 1997.

\bibitem{applegate2011traveling}
D.~L. Applegate, R.~E. Bixby, V.~Chv{\'a}tal, and W.~J. Cook, ``The traveling
  salesman problem,'' in \emph{The Traveling Salesman Problem}.\hskip 1em plus
  0.5em minus 0.4em\relax Princeton university press, 2011.

\bibitem{howard2002mobile}
A.~Howard, M.~J. Matari{\'c}, and G.~S. Sukhatme, ``Mobile sensor network
  deployment using potential fields: A distributed, scalable solution to the
  area coverage problem,'' in \emph{Distributed autonomous robotic systems
  5}.\hskip 1em plus 0.5em minus 0.4em\relax Springer, 2002, pp. 299--308.

\bibitem{kim2006local}
D.~H. Kim and S.~Shin, ``Local path planning using a new artificial potential
  function composition and its analytical design guidelines,'' \emph{Advanced
  Robotics}, vol.~20, no.~1, pp. 115--135, 2006.

\bibitem{paull2012sensor}
L.~Paull, S.~Saeedi, M.~Seto, and H.~Li, ``Sensor-driven online coverage
  planning for autonomous underwater vehicles,'' \emph{IEEE/ASME Transactions
  on Mechatronics}, vol.~18, no.~6, pp. 1827--1838, 2012.

\bibitem{li2020high}
P.~Li, C.-y. Yang, R.~Wang, and S.~Wang, ``A high-efficiency, information-based
  exploration path planning method for active simultaneous localization and
  mapping,'' \emph{International Journal of Advanced Robotic Systems}, vol.~17,
  no.~1, p. 1729881420903207, 2020.

\bibitem{silverman2013optimal}
Y.~Silverman, L.~M. Miller, M.~A. MacIver, and T.~D. Murphey, ``Optimal
  planning for information acquisition,'' in \emph{2013 IEEE/RSJ International
  Conference on Intelligent Robots and Systems}.\hskip 1em plus 0.5em minus
  0.4em\relax IEEE, 2013, pp. 5974--5980.

\bibitem{dressel_optimality_2018}
\BIBentryALTinterwordspacing
L.~Dressel and M.~J. Kochenderfer, ``On the optimality of ergodic trajectories
  for information gathering tasks,'' in \emph{2018 Annual American Control
  Conference ({ACC})}.\hskip 1em plus 0.5em minus 0.4em\relax {IEEE}, pp.
  1855--1861. [Online]. Available:
  \url{https://ieeexplore.ieee.org/document/8430857/}
\BIBentrySTDinterwordspacing

\bibitem{lasalle2016time}
J.~P. LaSalle \emph{et~al.}, ``The time optimal control problem,''
  \emph{Contributions to the theory of nonlinear oscillations}, vol.~5, pp.
  1--24, 2016.

\bibitem{romero2022time}
A.~Romero, R.~Penicka, and D.~Scaramuzza, ``Time-optimal online replanning for
  agile quadrotor flight,'' \emph{IEEE Robotics and Automation Letters},
  vol.~7, no.~3, pp. 7730--7737, 2022.

\bibitem{reynolds2001hybrid}
N.~Reynolds and P.~H. Meckl, ``Hybrid optimization scheme for time-optimal
  control,'' in \emph{Proceedings of the 2001 American Control Conference.(Cat.
  No. 01CH37148)}, vol.~5.\hskip 1em plus 0.5em minus 0.4em\relax IEEE, 2001,
  pp. 3421--3426.

\bibitem{chin1986optimum}
W.-P. Chin and S.~Ntafos, ``Optimum watchman routes,'' in \emph{Proceedings of
  the second annual symposium on Computational geometry}, 1986, pp. 24--33.

\bibitem{o1987art}
J.~O'rourke \emph{et~al.}, \emph{Art gallery theorems and algorithms}.\hskip
  1em plus 0.5em minus 0.4em\relax Oxford University Press Oxford, 1987,
  vol.~57.

\bibitem{tian2004effective}
L.~Tian and C.~Collins, ``An effective robot trajectory planning method using a
  genetic algorithm,'' \emph{Mechatronics}, vol.~14, no.~5, pp. 455--470, 2004.

\bibitem{wang2020robot}
W.~Wang, Q.~Tao, Y.~Cao, X.~Wang, and X.~Zhang, ``Robot time-optimal trajectory
  planning based on improved cuckoo search algorithm,'' \emph{IEEE access},
  vol.~8, pp. 86\,923--86\,933, 2020.

\bibitem{baghli2017optimization}
F.~Z. Baghli, Y.~Lakhal \emph{et~al.}, ``Optimization of arm manipulator
  trajectory planning in the presence of obstacles by ant colony algorithm,''
  \emph{Procedia Engineering}, vol. 181, pp. 560--567, 2017.

\bibitem{kim2015trajectory}
J.-J. Kim and J.-J. Lee, ``Trajectory optimization with particle swarm
  optimization for manipulator motion planning,'' \emph{IEEE transactions on
  industrial informatics}, vol.~11, no.~3, pp. 620--631, 2015.

\bibitem{du2022time}
Y.~Du and Y.~Chen, ``Time optimal trajectory planning algorithm for robotic
  manipulator based on locally chaotic particle swarm optimization,''
  \emph{Chinese Journal of Electronics}, vol.~31, no.~5, pp. 906--914, 2022.

\bibitem{tabak1971optimal}
D.~Tabak and B.~C. Kuo, \emph{Optimal control by mathematical
  programming}.\hskip 1em plus 0.5em minus 0.4em\relax SRL Publishing Company,
  1971.

\bibitem{kuhn1951nonlinear}
H.~Kuhn and A.~Tucker, ``Nonlinear programming in proceedings of 2nd berkeley
  symposium (pp. 481--492),'' \emph{Berkeley: University of California
  Press.[Google Scholar]}, 1951.

\bibitem{potra2000interior}
F.~A. Potra and S.~J. Wright, ``Interior-point methods,'' \emph{Journal of
  computational and applied mathematics}, vol. 124, no. 1-2, pp. 281--302,
  2000.

\bibitem{de2016ergodic}
G.~De~La~Torre, K.~Fla{\ss}kamp, A.~Prabhakar, and T.~D. Murphey, ``Ergodic
  exploration with stochastic sensor dynamics,'' in \emph{2016 American Control
  Conference (ACC)}.\hskip 1em plus 0.5em minus 0.4em\relax IEEE, 2016, pp.
  2971--2976.

\bibitem{posa2014direct}
M.~Posa, C.~Cantu, and R.~Tedrake, ``A direct method for trajectory
  optimization of rigid bodies through contact,'' \emph{The International
  Journal of Robotics Research}, vol.~33, no.~1, pp. 69--81, 2014.

\bibitem{shetty2021ergodic}
S.~Shetty, J.~Silv{\'e}rio, and S.~Calinon, ``Ergodic exploration using tensor
  train: Applications in insertion tasks,'' \emph{IEEE Transactions on
  Robotics}, vol.~38, no.~2, pp. 906--921, 2021.

\bibitem{mukkamala2017variants}
M.~C. Mukkamala and M.~Hein, ``Variants of rmsprop and adagrad with logarithmic
  regret bounds,'' in \emph{International conference on machine learning},
  2017, pp. 2545--2553.

\bibitem{duchi2011adaptive}
J.~Duchi, E.~Hazan, and Y.~Singer, ``Adaptive subgradient methods for online
  learning and stochastic optimization.'' \emph{Journal of machine learning
  research}, vol.~12, no.~7, 2011.

\bibitem{crazyswarm}
\BIBentryALTinterwordspacing
J.~A. Preiss*, W.~H\"onig*, G.~S. Sukhatme, and N.~Ayanian, ``Crazyswarm: {A}
  large nano-quadcopter swarm,'' in \emph{{IEEE} International Conference on
  Robotics and Automation ({ICRA})}.\hskip 1em plus 0.5em minus 0.4em\relax
  {IEEE}, 2017, pp. 3299--3304, software available at
  \url{https://github.com/USC-ACTLab/crazyswarm}. [Online]. Available:
  \url{https://doi.org/10.1109/ICRA.2017.7989376}
\BIBentrySTDinterwordspacing

\bibitem{doi:10.1126/scirobotics.abm6074}
\BIBentryALTinterwordspacing
S.~Macenski, T.~Foote, B.~Gerkey, C.~Lalancette, and W.~Woodall, ``Robot
  operating system 2: Design, architecture, and uses in the wild,''
  \emph{Science Robotics}, vol.~7, no.~66, p. eabm6074, 2022. [Online].
  Available: \url{https://www.science.org/doi/abs/10.1126/scirobotics.abm6074}
\BIBentrySTDinterwordspacing

\bibitem{zeng2021safety}
J.~Zeng, B.~Zhang, and K.~Sreenath, ``Safety-critical model predictive control
  with discrete-time control barrier function,'' in \emph{2021 American Control
  Conference (ACC)}.\hskip 1em plus 0.5em minus 0.4em\relax IEEE, 2021, pp.
  3882--3889.

\end{thebibliography}
\bibliographystyle{IEEEtran}

\begin{appendices} 
    \section{Proofs}
\label{appendix:proof}

\begin{proof} Theorem~\ref{thrm:1}:
    Using the Hamiltonian described in~\eqref{eq:ham}, we get the following time-optimal ergodic objective function
    \begin{align}\label{eq:obj_ham}
        \mathcal{J}(\bar{x},u,t_f, \rho, \lambda) &= \rho^\top\psi(\bar{x},t_f) \mid_{t_f} \nonumber \\ 
        &+ \int_{t_0}^{t_f} H(\bar{x},u,\lambda) - \lambda^\top \dot{\bar{x}} dt.
    \end{align}
    To find conditions of optimality over the continuous-time arguments of the objective function, we take the total variational derivative of~\eqref{eq:obj_ham}
    \begin{align}
        \delta \mathcal{J} &= \delta \rho^\top \psi(\bar{x},t_f)\mid_{t_f} + \bigg(\frac{\partial \psi}{\partial \bar{x}}^\top \rho \bigg)^\top \delta \bar{x} \bigg|_{t_f} \\ 
        & + \bigg(H-\lambda^\top \dot{\bar{x}} + \rho^\top \frac{\partial \psi}{\partial t_f}\bigg) \delta t_f \bigg|_{t_f} \nonumber \\
        &+ \int_{t_0}^{t_f} \frac{\partial H}{\partial \bar{x}}^\top \delta \bar{x} + \frac{\partial H}{\partial u}^\top \delta u - \lambda^\top \delta \dot{\bar{x}} + \bigg(\frac{\partial H}{\partial \lambda} - \dot{\bar{x}} \bigg)^\top \delta \lambda dt\nonumber
    \end{align}
    Using integration by parts, we can remove the terms $\lambda^\top \delta \dot{\bar{x}}$ and get the following total derivative 
    \begin{align}
        \delta \mathcal{J} &= \delta \rho^\top \psi(\bar{x},t_f)\mid_{t_f} 
        + \bigg(\frac{\partial \psi}{\partial \bar{x}}^\top \rho  - \lambda \bigg)^\top\delta \bar{x} \bigg|_{t_f} \\ 
        & + \bigg(H-\lambda^\top \dot{\bar{x}} + \rho^\top \frac{\partial \psi}{\partial t_f}\bigg) \delta t_f \bigg|_{t_f} \nonumber \\
        &+ \int_{t_0}^{t_f} \bigg(\frac{\partial H}{\partial \bar{x}}^\top  + \dot{\lambda} \bigg) \delta \bar{x} + \frac{\partial H}{\partial u}^\top \delta u + \bigg(\frac{\partial H}{\partial \lambda} - \dot{\bar{x}} \bigg)^\top \delta \lambda dt\nonumber.
    \end{align}
    Setting $\delta \mathcal{J} = 0$ we get the following system of equations 
    \begin{subequations}
    \begin{align}
        \psi(\bar{x}(t_f), t_f) = 0 \\ 
        \bar{x}(t_0) = \bar{x}_0 \\ 
        \dot{\bar{x}} = \frac{\partial H}{\partial \lambda}^\top \\ 
        \dot{\lambda} = - \frac{\partial H}{\partial \bar{x}}^\top \\ 
        \frac{\partial H}{\partial u} = 0 \\
        \lambda(t_f) =  \frac{\partial \psi}{\partial \bar{x}}^\top \rho \bigg|_{t_f} \\
        H(\bar{x}(t_f),u(t_f),\lambda(t_f), t_f) = -\rho^\top \frac{\partial \psi }{\partial t_f} \bigg|_{t_f}
    \end{align}
    \end{subequations}
    Recognizing the input stationarity condition $\frac{\partial H}{\partial u}$, we can reformulate the expression as an optimization over control constraints
    \begin{equation}
            u^\star = \argmin_{u \in \mathcal{U}} H(\bar{x}^\star, u, \lambda^\star).        
    \end{equation}
    Thus, we convert the time-optimal ergodic control problem into a point-wise optimization given locally optimal solutions to $\bar{x}^\star$ and $\lambda^\star$ where $\star$ denotes local optimality. 
    Therefore, the conditions of optimality for the time-optimal ergodic control problem are
    \begin{align*}
        \psi(\bar{x}(t_f), t_f) = 0 \\ 
        \bar{x}(t_0) = \bar{x}_0 \\ 
        \dot{\bar{x}} = \frac{\partial H}{\partial \lambda}^\top \\ 
        \dot{\lambda} = - \frac{\partial H}{\partial \bar{x}}^\top \\ 
        u^\star = \argmin_{u \in \mathcal{U}} H(\bar{x}^\star, u, \lambda^\star) \\       
        \lambda(t_f) =  \frac{\partial \psi}{\partial \bar{x}}^\top \rho \bigg|_{t_f} \\
        H(\bar{x}(t_f),u(t_f),\lambda(t_f), t_f) = -\rho^\top \frac{\partial \psi }{\partial t_f} \bigg|_{t_f}.
    \end{align*}
\end{proof}

\section{Implementation Details}
\label{appendix:implementation}

This appendix provides additional detailed information regarding parameters, initialization, environment configuration, and explicit equations used in the simulated and experimental results presented in Section~\ref{sec:results}.

\subsection{Simulated Results}

    All simulated results were done using 2-D point mass dynamics and a maximum number of basis $k_\text{max}=8$ e.g., $|\mathcal{K}^2|=64$.
    Note that the chosen dynamics is not specific to our implementation, as the proposed approach can handle arbitrary dynamics, but for ease of analysis. 
    % Solutions to Eq.~\eqref{eq:direct_time-opt-erg} are obtained using a gradient descent-based interior point solver~\cite{potra2000interior, kuhn1951nonlinear}. 

    \vspace{1.5mm}
    \noindent
    \textbf{A1}: 
    For results presented in Section~\ref{sec:results} A1, a bounded 2D exploration space $\expSpace = [0,1]\times [0,1]$ was used. 
    Here, the function $g:\mathcal{X} \to \expSpace$ is defined as $g(x) = \mathbf{I}_p x$ where $\mathbf{I}_p$ is an selection matrix that pulls the position terms from the state $x$.
    Initial conditions for the problem are given as $x_0 = [0.1, 0.1, 0,0]^\top$, $x_f=[0.9,0.9,0,0]^\top$. 
    The information distribution used is $\phi(w)=1$ which defines a uniform distribution. 
    Variations in $\gamma$ range from $0.005\to0.1$ spread evenly across $6$ values. 
    Control is constrained using $|u| \le u_\text{max}$ where $u_\text{max} = 1$ for both time-optimal and fixed-time ergodic problems.

    \vspace{1.5mm}
    \noindent
    \textbf{A2}: The bounded workspace for results presented in Section~\ref{sec:results} A2 are given as $\expSpace = [0,3.5]\times [-1,3.5]$. The function $g(x)$ then maps state to position in $\expSpace$. The solver is given as initial condition $x_0 = [1.5,-0.8,0,0]^\top$, $x_f = [2.0, 3.2, 0, 0]^\top$ and initial $t_f=10s$. 
    Here, we find that $N=100$ provided consistent trajectories. 
    Trajectories are found using $\gamma=0.1$ and $\gamma=0.001$ respectively. 
    Due to the larger space, the control constraint was set to be $u_\text{max}=2$.
    The distribution $\phi$ was specified as a mixture of Gaussian's 
    \begin{equation}
        \phi(w) = \sum_{i=0}^3 e^{- 10.5 \Vert w - c_i \Vert^2_2}
    \end{equation}
    where $c_i$ are Gaussian centers at $c_0= (1,-0.5),c_1=(2.5,0),c_2=(1.2, 2),c_3=(2.5,3)$.

    \vspace{1.5mm}
    \noindent
    \textbf{A3,4}: Ablation studies are performed under the same conditions as in A1. The initial parameters $t_f$ and $N$ are varied from $4-8s$ and $50-600$ with a resolution of $100$ after $50$. 
    The information distribution is given as $\phi(w)=1$ and we fix the ergodic upper bound $\gamma=0.05$.
    Control constraints are set as $u_\text{max}=1$ and the 2D point-mass dynamics are used.

\subsection{Solver Details}

    The solver used in this work is a variation of an augmented Lagrange solver~\cite{potra2000interior, kuhn1951nonlinear}. 
    The main augmentation is that we perform sub-gradient step using \cite{mukkamala2017variants, duchi2011adaptive}. 
    In practice, we found that this method did well in providing trajectory solutions that avoided saddle or inflection points. 

    \vspace{1.5mm}
    \noindent
    \textbf{Solver Parameters}:
        Initial parameters for the optimization problem in \eqref{eq:direct_time-opt-erg} varied depending on the ergodic upper bound $\gamma$. We found the following set of parameters would consistently yield solutions that were within solver tolerances. We used 200 knot points for $N$ and an $\alpha$ value for the CBF of 0.1. 
        Based on trial-and-error, we found that a linear interpolation from the initial to the final state as initial trajectory provided the most consistent solver performance. 
        Other variations may work (as shown in Fig.~\ref{fig:erg_init_ablation}), but due to the nonlinear landscape of the ergodic metric, it is difficult to predict the resulting solution. 
        For the three experimental results presented in Section~\ref{fig:exp_results}, we used an initial time of 8s, 12s, and 30s and ergodic upper bound $\gamma$ of 0.1, 0.05, and 0.01 respectively. 

\subsection{Results in Cluttered Environment}

    In this section, we provide the implementation details for the time-optimal ergodic search in a cluttered environment results in Section~\ref{sec:results:clutter}. 

    \vspace{1.5mm}
    \noindent
    \textbf{Aircraft Dynamics}: The exploration space $\expSpace$ is defined as a 3D space $\expSpace \subset \mathbb{R}^3$. Coverage trajectories are optimized over a uniform distribution $\phi(w)=1$ and all obstacles are known at optimization time. 
    The dynamics of the aircraft are given as 
    \begin{equation}
        \frac{d}{dt}\begin{bmatrix} x \\ y \\ z \\ \psi \\ \phi \\ v \end{bmatrix} = 
        \begin{bmatrix}
            v \cos \phi \cos \psi \\ 
            v \cos \phi \sin \psi \\ 
            v \sin \phi  \\ 
            u_1 \\ 
            u_2 \\ 
            u_3
        \end{bmatrix}
    \end{equation}
    where $x,y,z \in \expSpace$, and $u=[u_1, u_2,u_3]^\top$ are the control inputs with constraints $0.5\le u_1 \le 5$, and $|u_1|,|u_2| \le \frac{\pi}{3}$.
    
    \vspace{1.5mm}
    \noindent
    \textbf{System Configuration}:
    The exploration space $\expSpace$ is defined to be a 2D space with $\expSpace = [0,3.5]\times [-1,3.5]$. 
    All results used a uniform distribution $\phi(w)=1$. 
    Obstacles were placed in the environment at random and are fully known during planning time. 
    The dynamics used for planning are the 2D single-integrator system which we found mimic the drone control system reasonably. 
    The drone used was the Bitcraze Crazyflie 2.1 drone using the Lighthouse deck tracking system~\cite{crazyswarm} to obtain state position data. 
    The drone has an internal state estimator with inertial measurement unit that tracks velocity, orientation, and altitude. 
    Optimized trajectory are sent to the drone to track at the specified control frequency $\Delta t = t_f/N$ using ROS2 \cite{doi:10.1126/scirobotics.abm6074}.
    Experimental data was obtained through the Crazyflie library.

    \vspace{1.5mm}
    \noindent
    \textbf{Safety Barrier Constraints}:
        Safety and collision avoidance constraints for the obstacles in the environment were implemented using discrete control barrier functions~\cite{ames2019control, zeng2021safety, agrawal2017discrete}. 
        These functions define an inequality constraint that, when satisfied, guarantee a planned trajectories remain within a set of defined safe states. 
        The barrier constraint is defined as 
        % \begin{subequations}\label{eq:cbf-conds}
            \begin{align}
                &\Delta h_\text{CBF}(x_t, u_t) \geq - \alpha h_\text{CBF}(x_t)\\ 
                &\text{s.t. } 
                    \begin{cases}
                        0 < \alpha \leq 1 \\
                        \Delta h_\text{CBF}(x_t, u_t) =  h_\text{CBF}(f(x_t, u_t)) - h_\text{CBF}(x_t) \nonumber
                    \end{cases}
            \end{align}
        % \end{subequations}
        where $\alpha$ is a scalar value and $h_\text{CBF}(x)$ is given as the signed distance function to the centers of the rotated objects using a fourth order $L_4$ norm that smoothly approximates a square. 
        The constraint is implemented in~\eqref{eq:direct_time-opt-erg} for each knot discretization $N$.
    
\balance
\end{appendices}

\end{document}